\documentclass{article}

\usepackage[final]{neurips_2022}

\usepackage[utf8]{inputenc} % allow utf-8 input
\usepackage[T1]{fontenc}    % use 8-bit T1 fonts
\usepackage{hyperref}       % hyperlinks
\usepackage{url}            % simple URL typesetting
\usepackage{booktabs}       % professional-quality tables
\usepackage{amsfonts}       % blackboard math symbols
\usepackage{nicefrac}       % compact symbols for 1/2, etc.
\usepackage{microtype}      % microtypography
\usepackage{xcolor}         % colors
\usepackage{algorithm}
\usepackage{algpseudocode}
\usepackage{booktabs}
\usepackage{adjustbox}

\usepackage{amsmath,amssymb,graphicx, float, amsthm}
\usepackage[inline]{enumitem}
\usepackage{ifthen}
\usepackage[font={small}]{caption}
\usepackage{mathtools}

\DeclarePairedDelimiterX{\infdivx}[2]{(}{)}{%
  #1\;\delimsize\|\;#2%
}
\newcommand{\infdiv}{D\infdivx}

\theoremstyle{plain}
\newtheorem{thm}{Theorem}
\newtheorem{lem}[thm]{Lemma}
\newtheorem{crl}[thm]{Corollary}

\newboolean{show_comments}
\setboolean{show_comments}{true} 

\ifthenelse{\boolean{show_comments}}{
    \def\john#1{{\color{blue}\underline{\bf{John:}} #1}}
    \def\arnhav#1{{\color{blue}\underline{\bf{Arnhav:}} #1}}
    \def\arun#1{{\color{blue}\underline{\bf{Arun:}} #1}}
}{
    \def\john#1{}
    \def\arnhav#1{}
    \def\arun#1{}
}
\newboolean{if_appendix_failure}
\setboolean{if_appendix_failure}{true}   
\newboolean{if_appendix_algo}
\setboolean{if_appendix_algo}{true}
\newboolean{if_appendix_imposs}
\setboolean{if_appendix_imposs}{true}

\newcommand{\ER}{{Erd\H{o}s-R\'enyi }}
\newcommand{\RC}{{Rank-Centrality }}

\title{Byzantine Spectral Ranking}

\author{%
	Arnhav Datar \thanks{Supported in part by the Robert Bosch Centre for Data Science and Artificial Intelligence, IIT Madras.} \\
	Indian Institute of Technology, Madras\\
	\texttt{adatar@cmu.edu} \\
	\And
	Arun Rajkumar
	\footnotemark[1]\\
	Indian Institute of Technology, Madras\\
	\texttt{arunr@cse.iitm.ac.in} \\
	\AND
	John Augustine\thanks{Supported in part by the Cryptography, Cybersecurity, and Distributed Trust pCoE and the Accenture CoE at IIT Madras.} \\
	Indian Institute of Technology, Madras\\
	\texttt{augustine@iitm.ac.in} \\
}

\begin{document}

	\maketitle

	\begin{abstract}
	    %Rank aggregation from pairwise preferences has various applications in a wide spectrum of learning and social contexts such as social choice, recommendation systems, search, and information retrieval.
        We study the problem of rank aggregation where the goal is to obtain a global ranking by aggregating pair-wise comparisons of voters over a set of objects. We consider an adversarial setting where the voters are partitioned into two sets. The first set votes in a stochastic manner according to the popular score-based Bradley-Terry-Luce~(BTL) model for pairwise comparisons. The second set comprises malicious \emph{Byzantine} voters trying to deteriorate the ranking. We consider a strongly-adversarial scenario where the Byzantine voters know the BTL scores, the votes of the \emph{good} voters, the algorithm, and can collude with each other. We first show that the  popular spectral ranking based Rank-Centrality algorithm, though optimal for the BTL model, does not perform well even when a small constant fraction of the voters are Byzantine. 
	    
    	  We introduce the Byzantine Spectral Ranking Algorithm (and a faster variant of it), which produces a reliable ranking when the number of good voters exceeds the number of Byzantine voters. We show that no algorithm can produce a satisfactory ranking with probability $> 1/2$ for all BTL weights when there are more Byzantine voters than good voters, showing that our algorithm works for all possible population fractions. We support our theoretical results with experimental results on synthetic and real datasets to demonstrate the failure of the Rank-Centrality algorithm under several adversarial scenarios and how the proposed Byzantine Spectral Ranking algorithm is robust in obtaining good rankings.
	    
	\end{abstract}
	
    \section{Introduction}
    \label{sec:intro}

    Rank aggregation is a fundamental task in a wide spectrum of learning and social contexts such as social choice~\citep{social-choice1, social-choice2}, web search and information retrieval~\citep{search-1, search-2}, recommendation systems~\citep{recsys1} and crowd sourcing~\citep{crowd-source1}. Given pairwise comparison information over a set of $n$ objects, the goal is to identify a ranking that best respects the revealed preferences. Frequently, we are also interested in the scores associated with the objects to deduce the intensity of the resulting preference. There have been several solutions to this problem including Spectral Ranking methods such as \RC~\citep{negahban2017rank} and MLE methods~\citep{mle}.
    
    The Bradley-Terry-Luce~(BTL) model~\citep{btl, btl1} has been prevalent for a vast variety of ranking algorithms. The BTL model assumes that there are $n$ objects that are to be compared and each of these objects has a positive weight~($w_i$) associated with itself. Whenever any voter is asked a query for a pair, the voter claims $i$ is better than $j$ with probability ${w_i}/{(w_i + w_j)}$ independently at random. However, the BTL model assumes that all voters are identical and use the same probabilistic strategy for voting. Generally, in crowd-sourced settings, voters differ significantly from each other. Some of the voters may be spammers or even direct competitors of the entity ranking the objects, leading to data poisoning~\citep{sun2021data}. To model this situation, we assume that there is a split in the population. We consider that there are $K - F$ \emph{good} voters and $F$ \emph{Byzantine} voters, to appropriately model the division of the voters. The good voters vote according to the BTL model, while the Byzantine voters can potentially be controlled by an adversary that can centrally coordinate their behavior. 
    
        \subsection{Overview of our Results}
        \label{sec:contrib}
        
        We naturally initially consider the tolerance of the \RC algorithm against a constant fraction of Byzantine voters. \cite{negahban2017rank} showed that for an \ER pairwise-comparison graph $G(n,p)$ we have
        
        \begin{thm}[Informal from \cite{negahban2017rank}]
            If $p \in \Omega( \log n / n)$ and $k \in \Omega(1)$ then the \RC algorithm outputs a weight-distribution~($\pi$) such that with high probability\footnote{Whenever we say with high probability we mean with probability $\geq 1 - \frac{1}{n^c}$ for $c \geq 0$}
          $$ \frac{\lVert \pi - \tilde{\pi} \rVert}{\lVert \tilde{\pi} \rVert} \in \mathcal{O}\left(\sqrt{\frac{\log n}{knp}}\right)$$
        \end{thm}
            
        Here $\tilde{\pi}$ are the true weights of the objects and $k$ is the number of voter queries asked per edge in $G$. We can see that when $k \in \omega(1)$ the RHS goes to 0 as $n$ goes to $\infty$. We show that a strategy as simple as all Byzantine voters voting the reverse of the weights causes the \RC algorithm to output weights far from true weights with high probability. 
        
         \begin{thm}[Informal]
          There exists a strategy that the Byzantine voters can use such that the \RC algorithm outputs a distribution $\pi$ such that with high probability $$ \frac{\lVert \pi - \tilde{\pi} \rVert}{\lVert \tilde{\pi} \rVert} \in \Omega(1)$$
            provided there are a constant fraction of Byzantine voters
        \end{thm}
        
        Even simple algorithms that output the same weights regardless of the voter's responses also give $\Theta(1)$ error. To tackle this situation, we propose Byzantine Spectral Ranking, an algorithm that runs in polynomial time to remove some of the voters such that the \RC algorithm can be applied. As long as the proportion of Byzantine voters is bounded strictly below $1/2$, the algorithm removes 
        some voters such that the remaining Byzantine voters are unlikely to cause the \RC algorithm to deviate significantly.
        
        \begin{thm}[Informal]
            For a range of values of $k$, there exists an $\mathcal{O}(n^2)$-time algorithm that removes voters such that with high probability we have:
            \begin{equation*}
                 \frac{\lVert \pi - \tilde{\pi} \rVert}{\lVert \tilde{\pi} \rVert} \in  \mathcal{O} \left( \frac{d_{max}}{k} \right)        
            \end{equation*}
        \end{thm}
        
        Here $d_{max}$ is the maximum degree of the comparison graph. Finally, we also show that when $F \geq K/2$ there is no algorithm that can solve this problem for all weights with probability $> 1/2$, by coming up with a strategy that the Byzantine voters can use for two different weight distributions, and showing therefore that the algorithm must fail in at least one of the weight distributions. 
        
        \begin{thm}[Informal]
            If $F \geq K/2$, then there is no algorithm that can for all weights~($\tilde{\pi}$) with probability $ > 0.5$, give a weight distribution~($\pi^\ast$) such that $$\frac{\lVert {\pi}^\ast - \tilde{\pi} \rVert}{\lVert \tilde{\pi} \rVert} \in o(1)$$ 
        \end{thm}
        
        \subsection{Related Works}
        \label{sec:relworks}
        
        In recent times, there has been a lot of work on building robust machine learning algorithms~\citep{NIPS2017_f4b9ec30, Mhamdi2018TheHV, NEURIPS2020_d37eb50d, NEURIPS2021_d2cd33e9, 10.1145/3292040.3219655, yusenwu} that are resilient to Byzantine faults. 
        
        Ranking from pairwise comparisons has been a very active field of research and there have been multiple algorithms proposed with regard to the BTL model~\citep{negahban2017rank, pmlr-v80-agarwal18b, pmlr-v32-rajkumar14, spectralmle}. One of the most popular algorithms proposed in recent times has been the \RC algorithm. \cite{negahban2017rank} were able to show an upper-bound on the relative $L_2$-error~(between true weights and proposed weights) such that the relative $L_2$-error goes to 0 for larger $n$. They did this by finding the stationary distribution of matrix $P$ which was defined as
        \begin{equation}
            P_{ij} = \begin{cases}
                \frac{A_{ij}}{d_{max}} & \text{if } i\neq j\\
                1 - \frac{1}{d_{max}} \sum_{k \neq i} A_{ik} & \text{if } i= j
            \end{cases}
            \label{eqn:pfromadj}
        \end{equation}
        Here $A_{ij}$ is the number of times $i$ beats $j$ in $k$ queries divided by $k$, if $(i,j) \in E(G)$ else $A_{ij} = 0$. The pseudocode can be found in Algorithm \ref{alg:rankcent}. They also show that their algorithm outperformed other algorithms on various datasets. As our work builds on \cite{negahban2017rank} we have included a primer to \RC in the Appendix \ref{sec:rc_primer}.
        
        \begin{algorithm}
            \caption{\RC}
            \begin{algorithmic}[1]
                \Require $n$ objects to compare, $E$ the set of edges between these $n$ objects, $A$ the edge weights calculated based on the voter inputs.
                \Ensure A weighing of the $n$ objects such that the output weight is fairly close to input hidden weights.
                \State Compute the transition matrix P according to Equation \ref{eqn:pfromadj}
                \State Compute the stationary distribution $\pi$ (as the limit of $p^T_{t+1} = p^T_{t} P$). 

            \end{algorithmic}
            \label{alg:rankcent}
        \end{algorithm}

        Besides research on adversarial attacks in ranking, there has also been research on how the robustness of the ranking algorithms when exposed to noise. \cite{sync-rank} proposed Sync-Rank, by viewing the ranking problem as an instance of the group synchronization problem over the group $SO(2)$ of planar rotations. This led to a robust algorithm for ranking from pairwise comparisons, as was verified by synthetic and real-world data. \cite{Mohajer2016ActiveTR} worked on a stochastic noise model in order to find top-$K$ rankings. They were able to show an upper-bound on the minimum number of samples required for ranking and were able to provide a linear-time algorithm that achieved the lower-bound. However, both of these works~\citep{sync-rank, Mohajer2016ActiveTR} were primarily tackling the ranking problem when there was noisy and incomplete data as opposed to an adversary who was trying to disrupt the algorithm. 
        
        \cite{advtopk} worked on the adversarial ranking problem considering two cases one where the fraction of adversarial voters is known and is unknown. They got results that were asymptotically equivalent to~\cite{spectralmle}'s Spectral-MLE algorithm despite the presence of adversarial voters. However, they worked with the significantly weaker assumption that the adversarial voters will vote exactly opposite to the faithful voters (i.e. if the good voters voted that $i$ is better than $j$ with probability $\frac{w_i}{w_i+ w_j}$ then the Byzantine voters voted with probability $\frac{w_j}{w_i+ w_j}$). Our setting is much more general in that we do not make any assumptions regarding the Byzantine voters. 
        
        \cite{pmlr-v119-agarwal20a} worked on rank aggregation when edges in the graph were compromised. They were able to prove fundamental results regarding the identifiability of weights for both general graphs as well as \ER graphs. They also came up with polynomial-time algorithms for identifying these weights. However, their model is significantly different from ours. They assume that some of the comparisons in the comparison graphs might have been completely compromised while the others are free from any kind of corrtuption. In practice this is unlikely to happen, we assume a stronger model where every edge can have good and Byzantine voters. A detailed comparison with \cite{pmlr-v119-agarwal20a} made in section \ref{sec:agarwal}

    \section{Problem Setting}
    \label{sec:model}
    Given a set of $n$ objects, our goal is the obtain a ranking from a set of pairwise comparisons over the objects. We assume that the pairs to be compared are determined by an \ER comparison graph $G(n,p) = (V, E)$. Each edge in the graph corresponds to a pair that can be compared and we allow at most $k$ comparisons per pair. 
    
    The pairs can be compared by voters who fall into one of two categories - good or \emph{Byzantine}. Out of $K$ total voters, $F$ are assumed to be Byzantine. Every good voter when queried to compare a pair of objects, votes stochastically according to a Bradley-Terry-Luce (BTL) preference model parametrized by a weight/score vector $w \in \mathbb{R}^n_{+}$ where the probability of preferring object $i$ over object $j$ is given by $\frac{w_i}{w_i + w_j}$. The Byzantine voters on the other hand can vote however they wish (stochastic/deterministic). To make the problem general we consider a powerful central-adversary that could potentially control all of the Byzantine voters. We assume that the Byzantine adversary knows 
    \begin{enumerate*}[label=(\arabic*)]
        \item The queries asked to the good voters and their corresponding votes
       \item The comparison graph $G(n,p)$  
        \item The learning algorithm that is used to produce the ranking
        \item The underlying true BTL weights of the objects that the good voters use to produce their comparisons.
    \end{enumerate*}
    
    We assume that both the good voters and the byzantine voters will produce the same output when the same pair is asked to be compared multiple times by them.\footnote{This ensures that simple algorithms which query all voters to vote on the same pair multiple times and obtain a median preference will not work.}
    
    The goal of a learning algorithm is to use the least number of voter comparisons to output a score vector that is as close to the true scores. The learning algorithm  is allowed to choose any voter to compare any pair of objects.  However, it does not know the type of voter - good or Byzantine and naturally, a good algorithm must be robust to the presence of the powerful Byzantine voters. Finally, we work in the passive setting, where a learning algorithm is needed to fix the \emph{pair to voter} mapping apriori before the votes are collected.
    
    \section{Results for the \ER Graph}
    \label{sec:results}
    
    We focus our efforts on the \ER graph primarily because it has been heavily studied in the context of rankings. We pay careful attention to the $p \in \Theta(\log n / n)$. We go on to find that our algorithms turned out to be exponential in terms of the degree. This makes the \ER graph an optimal candidate for working out our algorithms because the degrees of the \ER graph will also be logarithmic with high probability, therefore giving us polynomial-time algorithms.
    % If $p$ happens to be in $\omega(\log n / n)$ we can simply discard some of the edges at random. 
    %Finally, for multiple analyses, we needed the graph to be close to regular i.e. we needed $d_{max} \in \mathcal{O}(d_{min})$, this also makes the \ER graph an ideal choice because with high probability \ER graphs are largely regular for large enough values of $p$. 
    
        \subsection{\RC Failure}
        \label{sec:fail}
        
        We initially try to motivate the need for a robust version of the \RC algorithm by showing that even a simple strategy from a Byzantine adversary will lead to an unsuitable ranking. In this scenario, the Byzantine adversary does not know the comparison graph, nor does it know the votes of good voters, and still can beat the \RC algorithm. Formally we show that:
        \begin{thm}
            Given $n$ objects, let the comparison graph be $G(n,p)$. Each pair in $G$ is compared $k$ times by voters chosen uniformly at random amongst the $K$ voters. If all Byzantine voters vote according to the opposite-weight strategy in an \ER graph with $p = C\log n/ n$ then there exists a $\tilde{\pi}$ for which the \RC algorithm will output a $\pi$ such that
            $$ \frac{\lVert \pi  - \tilde{\pi}\rVert}{\lVert \tilde{\pi} \rVert}\geq \frac{C\log n}{(2\sqrt{b} + b + 2bC\log n)} \cdot \frac{b-1}{8\sqrt{2}b(b+1)} \cdot \frac{F}{K}$$
           with probability $\geq 1 - \frac{1}{n^{knCF^2/64K^2}} - \frac{1}{n^{{C^2 \log n}/{8}}} - \frac{1}{n^{C/3 - 1}}$. Here $b$ is the skew in object weights i.e. $b = \max_{i,j} w_i / w_j$. 
        \end{thm}
        
        Here when we say the opposite-weight strategy we mean a strategy where the Byzantine voter will vote for $i$ if $w_i \leq w_j$ otherwise will vote for $j$. If we carefully examine the RHS, we see that the first term asymptotically converges to $\frac{1}{2b}$. The third term is also clearly a constant if there is a constant fraction of Byzantine voters. Therefore, showing the existence of a strategy that will with high probability, significantly deviate the weights given by the \RC algorithm. The detailed analysis can be found in Appendix \ref{sec:rc_fail_appendix}.
        
        \begin{proof}[Proof Sketch]
            We initially show that 
            
            \begin{equation}
                \label{eqn:failure_start}
            \frac{\lVert \pi - \tilde{\pi}\rVert}{\lVert \tilde{\pi} \rVert}
                    \geq \frac{\lVert \tilde{\pi}^T \Delta \rVert}{\lVert \tilde{\pi} \rVert \cdot (2\sqrt{b} + b\lVert \Delta \rVert_2)}    
            \end{equation}
            
            Here $\pi$ is the stationary distribution of the transition matrix~($P$) defined in Equation \ref{eqn:pfromadj} and $\Delta$ denotes the fluctuation of the transition matrix around its mean~($\tilde{P})$, such that $\Delta \equiv P - \tilde{P}$. We initially prove that $\lVert \Delta \rVert_2$ will at most be $1+d_{max}$, since there are at most $1+d_{max}$ terms in any row or column and each term can deviate by at most 1. Following this, we show that with probability $\geq 1 - \frac{1}{n^{C/3 - 1}}$ we will have $d_{max} \leq 2C \log n$. Finally, we show that $\lVert \tilde{\pi}^T \Delta \rVert$ will be $\in \Omega(pn)$ with probability $\geq 1 - \frac{1}{n^{knCF^2/64K^2}} - \frac{1}{n^{{C^2 \log n}/{8}}}$ by constructing a weight-distribution where half the weights are low and half were high. An arbitrary weight-distribution could not be used because if $\tilde{\pi} = \left[ \frac{1}{n}, \dots , \frac{1}{n} \right]$ we would get $\tilde{\pi}^T \Delta = 0$ for all $\Delta$. This is because the row sum of $P$ and $\tilde{P}$ will always be 1.
        \end{proof}

        \def \PROOFFAILURE{
        \begin{proof}
            \cite{negahban2017rank} define an inner product space $L^2(\tilde{\pi})$ as a space of $n$-dimensional vectors $\in \mathbb{R}^n$ endowed with
            $$ \langle a,b \rangle_{\tilde{\pi}} = \sum_{i=1}^n a_i \tilde{\pi}_i b_i.$$
            Similarly, they define $\lVert a \rVert_{\tilde{\pi}} = \langle a, a \rangle_{\tilde{\pi}}$ as the 2-norm in $L^2(\tilde{\pi})$ and $\lVert A \rVert_{\tilde{\pi}, 2} = \max_a \lVert Aa \rVert_{\tilde{\pi}} / \lVert a \rVert_{\tilde{\pi}}$ as the operator norm for a self-adjoint operator in $L^2(\tilde{\pi})$. They also relate these norms to norms in the Euclidean as
            \begin{align}
                \sqrt{\tilde{\pi}_{min}}\lVert a \rVert &\leq \lVert a \rVert_{\tilde{\pi}} \leq \sqrt{\tilde{\pi}_{max}}\lVert a \rVert, \label{normbound1}\\
                \sqrt{\frac{\tilde{\pi}_{min}}{\tilde{\pi}_{max}}} \lVert A \rVert_2 &\leq \lVert A \rVert_{\tilde{\pi},2} \leq \sqrt{\frac{\tilde{\pi}_{max}}{\tilde{\pi}_{min}}}\lVert A \rVert_2\label{normbound2}.
            \end{align}
            
            Finally \cite{negahban2017rank} define the following matrices.
            \begin{enumerate}[label=(\arabic*)]
                \item $\tilde{\Pi}$: Diagonal matrix such that $\tilde{\Pi}_{ii} = \tilde{\pi}_{i}$.
                \item $ S := \tilde{\Pi}^{1/2} \tilde{P} \tilde{\Pi}^{-1/2}$
                \item $S_1$: The rank-1 projection of $S$.
                \item $\tilde{P_1} := \tilde{\Pi}^{-1/2} \tilde{S_1} \tilde{\Pi}^{1/2}$
            \end{enumerate}

             Using the above results \cite{negahban2017rank} show that:
            \begin{equation}
            \label{eqn:rc14}
                (p_t - \tilde{\pi})^T = (p_{t-1} - \tilde{\pi})^T(\tilde{P} - \tilde{P_1} + \Delta) + \tilde{\pi}^T \Delta
            \end{equation}
                where $p_t$ is the distribution of the Markov chain on $P$ after $t$ iterations, for large $t$ we have $p_t \rightarrow \pi$. By the definition of $\tilde{P}_1$, it follows that $\lVert \tilde{P} - \tilde{P}_1 \rVert_{\tilde{\pi}, 2} = \lVert S - S_1 \rVert = \lambda_{max}$. Here $\lambda_{max}$ corresponds to the second-largest eigenvalue of the $\tilde{P}$ matrix, which by the Perron-Frobenius Theorem is $\lambda_{max} = \max (\lambda_2, |\lambda_n|)$. We can rearrange Equation \ref{eqn:rc14}  to get:
                \begin{align*}
                    \lVert \tilde{\pi}^T \Delta \rVert_{\tilde{\pi}} &\leq \lVert p_t - \tilde{\pi}\rVert_{\tilde{\pi}} + (\lVert \tilde{P} - \tilde{P}_1 \rVert_{\tilde{\pi}, 2} + \lVert \Delta \rVert_{\tilde{\pi}, 2}) \cdot \lVert p_{t-1} - \tilde{\pi}\rVert_{\tilde{\pi}}\\
                    &\leq \lVert p_t - \tilde{\pi}\rVert_{\tilde{\pi}} + \rho \cdot \lVert p_{t-1} - \tilde{\pi}\rVert_{\tilde{\pi}}
                \end{align*}
                here $\rho = \lambda_{max} + \lVert \Delta \rVert_{\tilde{\pi},2} $. Choosing a large enough $t$ where the Markov chain has converged we get:
                \begin{align*}
                    \lVert \tilde{\pi}^T \Delta \rVert_{\tilde{\pi}} 
                    &\leq (1 + \rho) \cdot \lVert \pi - \tilde{\pi}\rVert_{\tilde{\pi}} 
                \end{align*}

                 We then use the bounds in Equations \ref{normbound1} and \ref{normbound2} to get:
                \begin{align*}
                    \frac{\lVert \pi - \tilde{\pi}\rVert_{\tilde{\pi}}}{\lVert \tilde{\pi} \rVert} 
                    &\geq \frac{\lVert \tilde{\pi}^T \Delta \rVert_{\tilde{\pi}}}{\lVert \tilde{\pi} \rVert \cdot (1 + \rho)}\\
                    \sqrt{\tilde{\pi}_{max}} \cdot \frac{\lVert \pi  - \tilde{\pi}\rVert}{\lVert \tilde{\pi} \rVert} 
                    &\geq \sqrt{\tilde{\pi}_{min}} \frac{\lVert \tilde{\pi}^T \Delta \rVert}{\lVert \tilde{\pi} \rVert \cdot (1 + \rho)}\\
                    \frac{\lVert \pi  - \tilde{\pi}\rVert}{\lVert \tilde{\pi} \rVert} 
                    &\geq \sqrt{\frac{\tilde{\pi}_{min}}{\tilde{\pi}_{max}}} \frac{\lVert \tilde{\pi}^T \Delta \rVert}{\lVert \tilde{\pi} \rVert \cdot (1 + \lambda_{max} + \sqrt{\frac{\tilde{\pi}_{max}}{\tilde{\pi}_{min}}}\cdot \lVert \Delta \rVert_2)}
                \end{align*}
                Since $\lambda_{max} \leq 1$ we get:
                \begin{equation}
                \label{eqn:counter_example_2}
                    \frac{\lVert \pi  - \tilde{\pi}\rVert}{\lVert \tilde{\pi} \rVert}
                    \geq \frac{\lVert \tilde{\pi}^T \Delta \rVert}{\lVert \tilde{\pi} \rVert \cdot (2\sqrt{b} + b\lVert \Delta \rVert_2)}
                \end{equation}
                Using lemmas \ref{lem:denom_bound} and \ref{lem:num_bound} in Equation \ref{eqn:counter_example_2} we can conclude that :
                \begin{equation}
                    \frac{\lVert \pi  - \tilde{\pi}\rVert}{\lVert \tilde{\pi} \rVert}
                    \geq \frac{C\log n}{(2\sqrt{b} + b + 2bC\log n)} \cdot \frac{b-1}{8\sqrt{2}b(b+1)} \cdot \frac{F}{K}
                \end{equation}
               with probability $\geq 1 - \frac{1}{n^{C/3 - 1}}  -  \frac{1}{n^{knCF^2/64K^2}} - \frac{1}{n^{{C^2 \log n}/{8}}} $. We can see that for large enough $n$ the RHS will converge to a constant with high probability.
            \end{proof}

        }
        \def\PROOFFAILURELEMMAS{
            \begin{lem}
                \label{lem:denom_bound}
                For the opposite-voting Byzantine strategy, we have
                $$\lVert \Delta \rVert_2 \leq 1+2C \log n$$
               with probability $\geq 1 -  \frac{1}{n^{C/3 - 1}}$
            \end{lem}
            \begin{proof}
                
                We will first get an upper-bound for the denominator. We know that:
                \begin{align*}
                    \lVert \Delta \rVert_2 &\leq \sqrt{\lVert \Delta \rVert_1 \cdot \lVert \Delta \rVert_{\infty}}
                \end{align*}
                These are essentially the maximum row and column sums. Since all of the entries in $\Delta$ are the difference corresponding entries between $P$ and $\tilde{P}$ we know that they will always be between $-1$ and $1$. Therefore we can say that both the maximum column and maximum row sum will be at most $1+ d_{max}$. We can see the probability of having a vertex with a degree greater than $2C \log n$ can by Chernoff's bound that $\mathbb{P}(d_i \geq 2C \log n) \leq e^{-C\log n /3}$. Applying union bound we can say that $\mathbb{P}(d_{max} \geq 2C \log n) \leq ne^{-C\log n /3}$. We can see that by setting the value of $C \geq 3$ we can say that with high probability that $d_{max} \leq 2C\log n$. Therefore we can conclude that $\lVert \Delta \rVert_2 \leq 1+2C \log n$. Hence Proved.
                \end{proof}
                
            \begin{lem}
                \label{lem:num_bound}
                For the opposite-voting Byzantine strategy, there exists a $\tilde{\pi}$ such that we have
               $$ \frac{\lVert \tilde{\pi}^T \Delta \rVert}{\lVert \tilde{\pi} \rVert} \geq \frac{(b-1)C \log n}{8\sqrt{2}b(b+1)} \cdot \frac{F}{K}$$
               with probability $\geq 1 -  \frac{1}{n^{knCF^2/64K^2}} - \frac{1}{n^{{C^2 \log n}/{8}}}$
            \end{lem}
            \begin{proof}
                We now try to get a lower-bound on the RHS of Equation \ref{eqn:counter_example_2}. We consider weights:
                \begin{equation*}
                    \tilde{\pi}_i = \begin{cases}
                        \frac{1}{n + (b - 1) \cdot \lceil \frac{n}{2} \rceil } & \text{for }1 \leq i \leq \lfloor \frac{n}{2} \rfloor\\
                        \frac{b}{n + (b - 1) \cdot \lceil \frac{n}{2} \rceil } & \text{for }\lfloor \frac{n}{2} \rfloor < i \leq n 
                    \end{cases}
                \end{equation*}
                
                Let us call the resulting vector from $\tilde{\pi}^T \Delta$ as $v$. We can see that for $i = 1$ to $\lfloor n/2 \rfloor$  we have:
                \begin{align*}
                    v_i &= \sum_{(i,j) \in E} \tilde{\pi}_j \Delta_{ji} + \tilde{\pi}_i \Delta_{ii}\\
                    &= \sum_{(i,j) \in E \& j\leq \lfloor n/2 \rfloor} \frac{\Delta_{ji}}{n + (b - 1) \cdot \lceil \frac{n}{2} \rceil }   +
                    \sum_{(i,j) \in E \& j> \lfloor n/2 \rfloor} \frac{b \Delta_{ji}}{n + (b - 1) \cdot \lceil \frac{n}{2} \rceil }  + \tilde{\pi}_i \Delta_{ii}\\
                    &= \frac{(b-1)}{n + (b - 1) \cdot \lceil \frac{n}{2} \rceil } \sum_{(i,j) \in E \&  j> \lfloor n/2 \rfloor} \Delta_{ji} \tag{$\sum_{(i,j) \in E} \Delta_{ji} + \Delta_{ii} = 0$}
                \end{align*} 
                Given the above relation for $v_i$ we can say that
                \begin{align*}
                    \sum_{i=1}^{\lfloor n/2 \rfloor} v_i^2 &= \left( \frac{(b-1)}{n + (b - 1) \cdot \lceil \frac{n}{2} \rceil } \right)^2  \cdot \left( \sum_{i=1}^{\lfloor n/2 \rfloor}\left(\sum_{(i,j) \in E \&  j> \lfloor n/2 \rfloor} \Delta_{ji} \right)^2 \right)\\
                    &\geq \left( \frac{(b-1)}{n + (b - 1) \cdot \lceil \frac{n}{2} \rceil } \right)^2 \cdot \frac{1}{\left\lfloor \frac{n}{2} \right\rfloor} \cdot \left( \sum_{i=1}^{\lfloor n/2 \rfloor}\sum_{(i,j) \in E \&  j> \lfloor n/2 \rfloor} \Delta_{ji} \right)^2 \tag{Cauchy-Schwarz}
                \end{align*}
                
                Since the first two terms are already known we are primarily interesting in getting a high probability bound on the third term. We can see that we are only interested in $\lfloor n/2 \rfloor \cdot \lceil n/2 \rceil \sim n^2/4$ entries in the adjacency matrix. We can see that each of these have a probability $p$ of being an edge. Since $p = C\log n/n$ we can conclude there will be at least $pn^2/8$ edges with probability $$\geq 1 - e^{-\frac{p^2n^2}{8}} \geq 1 - \frac{1}{n^{{C^2 \log n}/{8}}}$$

                Furthermore, it is easy to see a single byzantine voter amongst the $k$ voters will cause a deviation of $\frac{b}{(b+1)k}$. On the other hand, any good voter will in expectancy not cause any deviation. Therefore in expectancy, we can say that the deviation for a randomly chosen voter will be $\frac{Fb}{(b+1)kK}$. Here we are summing $k\cdot pn^2/8$ binomial random variables, therefore we can say that by Chernoff bound out of the $kpn^2/8$ voters the total deviation will be at least $\frac{bF}{2Kk(b+1)} \cdot \frac{kpn^2}{8} $ Byzantine voters with probability $$ \geq 1 - e^{- \frac{F^2}{64K^2} \cdot Ckn\log n}  = 1 - \frac{1}{n^{knCF^2/64K^2}}$$
                
                Therefore we can claim that:
                \begin{align*}
                    \left( \sum_{i=1}^{\lfloor n/2 \rfloor}\sum_{(i,j) \in E \&  j> \lfloor n/2 \rfloor} \Delta_{ji} \right) &\geq \frac{Fkpn^2}{16K} \cdot \frac{b}{(b+1)k}\\
                    &\geq \frac{FCnb\log n}{16K(b+1)} 
                \end{align*}
                Substituting in the above equation and considering larger $n$ we get:
                \begin{align*}
                    \sum_{i=1}^{\lfloor n/2 \rfloor} v_i^2 &\geq \frac{2}{n} \cdot \left(\frac{b-1}{nb}
                     \cdot \frac{FCnb\log n}{32K(b+1)} \right)^2
                \end{align*}
                Coming back to our objective of bounding $\lVert \tilde{\pi}^T \Delta \rVert$ we can say that:
                $$\lVert \tilde{\pi}^T \Delta \rVert \geq  \frac{(b-1)FC\log n}{16(b+1)K} \cdot  \sqrt{\frac{1}{2n}}$$
                
                We can also see that the maximum term in $\tilde{\pi}$ can at most be $\frac{b}{b+n-1}$. Therefore we can conclude that $\lVert \tilde{\pi} \rVert$ will at most be $\sqrt{n} \cdot \frac{b}{b+n-1}$, which for constant values of $b$ essentially comes out to $\leq \frac{b}{2\sqrt{n}}$. 
                
                Therefore we claim that:
                $$ \frac{\lVert \tilde{\pi}^T \Delta \rVert}{\lVert \tilde{\pi} \rVert} \geq \frac{(b-1)C \log n}{8\sqrt{2}b(b+1)} \cdot \frac{F}{K} $$
               
               with probability $\geq 1 - \frac{1}{n^{knCF^2/64K^2}} - \frac{1}{n^{{C^2 \log n}/{8}}}$
            \end{proof}
        }
        
        \ifthenelse{\boolean{if_appendix_failure}}{}{\PROOFFAILURE \PROOFFAILURELEMMAS} 
        
        \subsection{Byzantine Spectral Ranking}
        \label{sec:alg}
        
        In this section, we propose a novel algorithm called Byzantine Spectral Ranking~(BSR). The BSR algorithm relies on asking multiple voters multiple queries and based on the collective response decides whether a voter is acting suspiciously or whether it is plausible that the voter might be a good voter. For each object $i$ to be compared, we take a set of $d_i$ objects to compare it against and have $k$ voters compare $i$ with all of the $d_i$ objects~(leading to $kd_i$ comparisons for each voter). Based on the voter's responses, we try to eliminate voters who are likely to be Byzantine. The BSR algorithm ensures that a constant fraction of good voters remain and that if any Byzantine voters remain then these voters have votes similar to good voters. Intuitively, the BSR algorithm goes over every single subset of $[d_i]$ and tries to find if the Byzantine voters have collectively voted strongly in favor or against this subset. The pseudocode can be found in Algorithm \ref{alg:byzalg1}. 
        
        % We come up with the algorithm by reverse-engineering the analysis of the \RC algorithm where we had to show the sum of row deviations of the $\Delta$ matrix was bounded. To this effect, we come up with the $\mathsf{Bound\_Sum\_Deviations}$. The exact mechanisms of how this function works can be found in Sections \ref{sec:algana} and \ref{sec:alganalysis}.
        
            \begin{algorithm}
                \caption{Byzantine Spectral Ranking}
                \begin{algorithmic}[1]
                    \Require $n$ objects, $K$ voters, Comparison graph $G$, $k$ comparisons allowed per edge, parameter $Q$.
                    \Ensure Weights~($\pi$) for each of the objects in $[n]$
                    \Procedure{$\mathsf{Bound\_Sum\_Deviations}$}{$i, S, d, k, \alpha$}
                        \State Pick $k$ voters randomly from the voter set~($[1, \dots, K])$
                        \State The $k$ voters are queried to compare all objects in $S$ with $i$
                        \Comment{$k \times d$ queries in total}
                        \State $T$ represents the binary representation of the voters' inputs
                        \Comment{$T$ will be $k \times {d}$ matrix}
                        \State $\delta \gets \sqrt{\frac{Q}{2} d \log k}$ 
                        \Comment{We show later that $Q$ needs to be $\geq 1$}
                        \State $M \gets [0]_{2^{d} \times k}$
                        
                        \For {each $\xi \in \{1,-1\}^{d}$}
                            \State $U \gets T\xi$
                            \State $\hat{m} \gets \mathsf{Median}(U)$
                            \State $M[\xi][v] \gets 1$ if voter $v$ is $5\delta$ away from $\hat{m}$ and 0 otherwise
                        \EndFor
                        
                        \State $\mathsf{max\_out} \gets 8k^{1-Q} + {8k^{1-\alpha}}$
                        \State Remove voters with $M[\xi][v] = 1$ if the sum of $M[\xi]$ is $\geq \mathsf{max\_out}$ for some $\xi$
                        \State Return leftover $\mathsf{votes}$
                    \EndProcedure
                    \State $\alpha \gets 1 - {\log d_{max}}/ {\log k} $ \Comment{$\alpha$ is a parameter whose value is set during analysis}
                    \For {each object $i$ to be compared}
                        \State $\mathsf{votes}  \gets \mathsf{Bound\_Sum\_Deviations}$($i$, $N_G(i)$, $d_i$, $k$, $\alpha$)
                        \State Using $\mathsf{votes}$ update $P$ as described in Equation \ref{eqn:pfromadj} 
                    \EndFor
                    \State Compute the stationary distribution $\pi$ which is the limit of $p^T_{t+1} = p^T_{t} P$.
                \end{algorithmic}
                \label{alg:byzalg1}
            \end{algorithm}
            
            \subsubsection{Analysis}
            \label{sec:algana}
        
            The detailed analysis can be found in Appendix \ref{sec:alganalysis}. Based on the proof of the \RC Convergence the only result which does not hold for Byzantine voters is Lemma 3, where they bound the 2-norm of the $\Delta$ matrix. We show that: 
        
            \begin{lem}
                For an \ER graph with $p \geq 10C^2 \log n / n$ with $k \in \Omega(np)$ comparisons per edge, algorithm \ref{alg:byzalg1} removes some voters such that we have:
                \begin{equation*}
                    \lVert \Delta \rVert_2 \leq 80 \max\left( \frac{d_{max}}{k}, \sqrt{\frac{\log k}{d_{min}}}\right)
                \end{equation*}
                with probability $\geq 1 - 16n^{(1-1.5C^2)}$.
                \label{thm:byz1}
            \end{lem}

            \begin{proof}[Proof Sketch]
                 Proving that $\lVert \Delta \rVert_2$ is bounded with high probability turns out to be equivalent to proving that $\mathbb{P}( \sum_{j\in \partial i} |C_{ij}| > kd_is )$ is bounded. $C_{ij}$ is defined as $kA_{ij} - kp_{ij}$, where $p_{ij} = \frac{w_i}{w_i + w_j}$. Since it is a sum of $d_i$ absolute values it can be upper-bounded by considering the union bound over all $2^{d_i}$ combinations for this sum.
                
                We initially show that with high probability for all $2^{d_i}$ values of $\xi$ the mean($\hat{m}$) that we compute in the BSR algorithm will only be $\mathcal{O}(\delta)$ away from the expected mean. This gives us a good check for whether a voter is acting out of the ordinary or not. The algorithm then filters out voters that are collectively acting out of the ordinary for a particular $\xi$. We show that at the end of these removals, the number of voters will only be a constant factor far from $k$, i.e. most of the good voters will be retained. Finally, we show an upper-bound on the row sum of the absolute values of the deviations of the remaining Byzantine voters to complete the proof.\end{proof}
                
            Intuitively, we get the terms ${d_{max}}/{k}$ and $\sqrt{{\log k}/{d_{min}}}$ in Lemma \ref{thm:byz1} by considering the maximum average deviation for the row sum that can be caused for any $\xi$ by the Byzantine voters. This can be written as $\mathsf{max\_out}$ voters having a deviation of $d$ and the remaining $k$ voters having a deviation of $\delta$. Therefore we can see that the overall deviation will be $\frac{1}{kd}\left( \mathsf{max\_out} \cdot d  + \delta \cdot k\right)$. Substituting the values of the parameters, and choosing an appropriate $Q$ we get $\mathcal{O}\left( \max \left( \frac{d_{max}}{k} , \sqrt{\frac{\log k}{d_{min}}} \right)\right)$. Finally, using lemma \ref{thm:byz1}, and a few lemmas from \cite{negahban2017rank} we conclude that:
            
            \begin{thm}
                Given $n$ objects, let the comparison graph be $G(n,p)$. Each pair in $G$ is compared $k$ times with the outcomes of comparisons produced as per the Byzantine-BTL model with weights $[w_1 , \dots , w_n]$. Then there exists an algorithm such that when $F \leq K(1 - \epsilon)/2$, $p \geq 10C^2 \log n / n$, $\epsilon > 0$ and $k \geq 18d_{max} / \epsilon^2$ the following bound on the error 
                    $$\frac{\lVert \pi - \tilde{\pi} \rVert}{\lVert \tilde{\pi} \rVert} \leq 480b^{5/2}  \max \left( {\frac{d_{max}}{k}}, \sqrt{\frac{\log k}{d_{min}}} \right)$$
                holds with probability $\geq 1 - 6n^{-C/8} -16n^{(1-1.5C^2)}$.
                \label{thm:byz2}
            \end{thm}
            
            \textbf{Remark. } The BSR algorithm gives a suitable ranking in polynomial time, but is fairly slow as the complexity will be $\in \Omega(n^{5C^2+1})$, which is not suitable for larger values of $n$ and $C$.
            
            \textbf{Remark. } Our process of filtering out Byzantine voters is tightly tied to the analysis of the \RC algorithm and the procedure to remove voters may not be suitable for other algorithms. The advantage of tying the removal procedure to the ranking algorithm is that we would not remove voters who may vote in an adversarial fashion, but do not significantly affect the row sum of the absolute difference between the empirical and true Markov chain transition matrices. An algorithm that tries to explicitly identify Byzantine voters might end up removing such voters too (and hence their votes for other useful pairs) which might be counter-productive.
            
            \def \PROOFALG{
            \begin{proof}
            We initially prove a few lemmas to find out the probability that for some $\xi$ a voter will be at a distance of $\delta$ away from the mean.
            
            \begin{lem}
                For each good voter $v$ and for each $\xi \in  \{ 1, -1\}^d $ let $D_{i,v} = \sum_{j \neq i} \xi_{j} \cdot y_{i,j,v}$ where $y_{i,j,v}$ is the vote casted by voter $v$ when queried about $(i,j)$. Then $D_{i,v}$ is at distance $ \geq \delta$ away from the expected mean with probability at most  $2 \exp\left( - \frac{2 \delta^2}{d_i} \right)$.
                \label{lem:dist}
            \end{lem}
            \begin{proof}
                We can directly apply Hoeffding's inequality as even after being multiplied by $\xi$ the range of $y_{i,j,v}$ still remains between $-1$ and $+1$. Therefore we can conclude that with probability $\geq 2 \exp\left( - \frac{2 \delta^2}{d_i} \right)$, $D_{i,v}$ does not deviate from the mean by more than $\delta$.
            \end{proof}
            \begin{crl}
                When $\delta \geq \sqrt{\frac{Q}{2} d_i \log k}$, then we have the probability of being at most $\delta$ away from the mean $\geq 1 - 2e^{-Q\log k} = 1 - 2k^{-Q}$.
                \label{cor:slight}
            \end{crl}
            \begin{proof}
                Trivially true using Lemma \ref{lem:dist}.
            \end{proof}

            \begin{lem}
                If $\delta =  \sqrt{\frac{Q}{2} \cdot d_i \log k}$ and $F \leq \frac{K(1-\epsilon)}{2}$ then with probability $$\geq 1 -  \exp\left(-\frac{k\epsilon}{18} + d_i \log 2 \right) -  \exp\left(-\frac{k\epsilon^2}{2} \right)$$
                the predicted mean will be $\mathcal{O}(\delta)$ away from the observed mean for all $\xi$. 
            \end{lem}
            \begin{proof}
                We define a voter to be in good range if it is $\delta$ away from the mean, otherwise, we say it is in bad range. We now make two claims:
                \begin{enumerate}[label=(\arabic*)]
                    \item \label{clm1} At least $(1 - \epsilon_1)$ fraction of the good voters will be in good range with a probability $\geq 1 - \exp(-k\epsilon_1 / 6 + d_i \log 2)$ for all values of $\xi$. 
                    \item \label{clm2}  With a probability of $\geq 1 - \exp(-\frac{k\epsilon^2}{2})$ we will have $f < \frac{k}{2}\cdot (1-\frac{\epsilon}{2})$, where $f$ are the number of Byzantine voters amongst the $k$ chosen voters.
                \end{enumerate}
                
                \textbf{Proof of Claim \ref{clm1}:} From Corollary \ref{cor:slight} we know that a voter will be in good range with probability $\geq 1 - 2k^{-Q}$. Therefore we can claim by the Chernoff bound(setting $\mu = k^{1-Q} /2$ and $\delta = 2\epsilon_1 k^Q - 1$) that for some $\xi$ the probability of having $(1 - \epsilon_1)$ voters not being in good range can be written as $ \geq 1 - \exp\left( k\epsilon_1 / 6 \right)$. However, for this to be applicable over all values of $\xi$ we need to multiply this by $2^{d_i}$. Therefore we can get the statement of Claim \ref{clm1}. It should be noted that for Claim 1 to hold with high probability we need
                $$k \geq \frac{6}{\epsilon_1}(c^\ast \log n + d_{max} \log 2)$$
                
                \textbf{Proof of Claim \ref{clm2}:} We can also see that since $F \leq \frac{K(1-\epsilon)}{2} $ we can apply Hoeffding's inequality to conclude the statement of claim 2. 
                
                By claim \ref{clm1}, we can see that there are at least $(1 - \epsilon_1) \cdot (k-f)$ good voters that lie in good range with probability $\geq 1 -  \exp(-k\epsilon_1 / 6 + d_i \log 2)$. By Claim \ref{clm2}, we can  also see that $(k-f)$ will be at least  $\frac{k}{2}\cdot (1+\frac{\epsilon}{2})$ with probability $\geq 1 - \exp(-\frac{k\epsilon^2}{2})$. Therefore, we can conclude that there are at least $\frac{k}{2} \cdot (1 - \epsilon_1) \cdot (1+\frac{\epsilon}{2})$ good voters in good range with probability $1 -  \exp(-\frac{k\epsilon^2}{2}) -  \exp(-k\epsilon_1 / 6 + d_i \log 2)$ by using a simple application of the union bound. By choosing the value of $\epsilon_1$ such that
                \begin{equation}
                \label{eqn:epsrel}
                    \epsilon_1 \leq \frac{\epsilon}{2 + \epsilon} 
                \end{equation}
                  We can see that at most $k/2$ good voters will be in good range whp if $k \in \Omega(\log n)$. Therefore we can set $\epsilon_1 = \epsilon/3$ to satisfy Equation \ref{eqn:epsrel}.
                
                Since know that at least $k/2$ of the voters lie in the good range we can conclude that the median of the voter's~(good and Byzantine) outputs also lies in the good range. Therefore we can conclude that the sample median of the good voters is at most $2\delta$ away from the predicted median. 
                
                While we have this result what we ideally want is a good predictor for the expected mean. We have already shown that the predicted median would be at a distance $2\delta$ from the sample median of the good voters. We will now go on to show that the sample median of the good voters is also very close to the expected mean. For the sample median to be more than $2\delta$ away from the expected mean we would need there to be less than $50\%$ of the entries $\delta$ away from the expected mean. However, since we are assuming Claim 1 is in effect we know that $50\%$ of the entries are in good range with high probability. We can therefore conclude that the predicted median is at most $4\delta$ away from the expected mean. Therefore if we consider entries that are $5\delta$ away from the predicted median this will include the entries that are $\delta$ away from the mean encapsulating a large fraction of the voters while ensuring that any chosen voter is at most $9\delta$ away from the expected mean. \end{proof} 
                
                % While we have this result what we ideally want is a good predictor for expected mean. We have already shown that the predicted median would be at a distance $2\delta$ from the sample median. We will now go on to show that the sample median is also very close to the expected median. For the sample median to be more than $2\delta$ away from the expected median we would need there to be less than $50\%$ of the entries $\delta$ away from the expected median. However, since we are assuming Claim 1 is in effect we know that $50\%$ of the entries are in good range with high probability. 
                
                % Furthermore, since we know that for a binomial distribution the median and mean are separated by at most $1$, we can conclude that the predicted median is at most $1 + 4\delta$ away from the expected mean. Therefore if we consider entries that are $1+5\delta$ away from the predicted median this will include the entries that are $\delta$ away from the mean encapsulating a large fraction of the voters while ensure that the any voter chosen is at most $9\delta + 2$ away from the expected mean. \end{proof}  

            We will now finally prove Lemma \ref{thm:byz1}. As in the \RC algorithm's analysis, we will prove this in two parts. We know that $\Delta = D + \bar{\Delta}$, where $D$ is a diagonal matrix consisting only of the diagonal entries of $\Delta$. We can also decipher that $\lVert  \Delta \rVert_2 \leq \lVert  D \rVert_2 + \lVert \bar{\Delta}\rVert_2$ We now separately show that both $\lVert  D \rVert_2$ and $\lVert \bar{\Delta}\rVert_2$ are $o\left(1 \right)$.
            
            \begin{lem}
                Algorithm \ref{alg:byzalg1} removes some voters such that we have:
                \begin{equation*}
                    \lVert \bar{\Delta} \rVert_2 \leq 40 \max \left( \frac{d_{max}}{k} , \sqrt{\frac{\log k}{d_{min}}} \right)
                \end{equation*}
                with probability $\geq 1 - \frac{8}{n^{1.5C^2-1}}$. 
                \label{lem:byz_non_diag}
            \end{lem}
            
            \begin{proof}
                
                We will initially show that the non-diagonal matrix is bounded. Since we know that:
                \begin{equation*}
                    \lVert \bar{\Delta} \rVert_2 \leq \sqrt{\lVert \bar{\Delta} \rVert_1 \lVert \bar{\Delta} \rVert_\infty}
                \end{equation*}
                Since $\bar{\Delta}$ is a symmetrical matrix we can simply bound the maximal-row sum of the absolute values of $\bar{\Delta}$. To that end we define 
                $$R_i = \frac{1}{kd_{max}} \sum_{j \neq i} |C_{ij}|$$
                where $C_{ij} = kA_{ij} - kp_{ij}$. Continuing from Equation (22) from \cite{negahban2017rank} we get:
                \begin{align}
                    \mathbb{P}(R_i \geq s) &= \mathbb{P}( \sum_{j\in \partial i} |C_{ij}| > kd_is )\\
                    &\leq \sum_{\xi \in \{-1, 1\}^{d_i}} \mathbb{P}\left(\sum_{j \in \partial i} \xi_j C_{ij} > kd_is \right) \\
                    &\leq \sum_{\xi \in \{-1, 1\}^{d_i}} \mathbb{P}\left(\sum_{j \in \partial i\text{ and good}} \xi_j C_{ij} > \frac{kd_is}{2} \right) + \mathbb{P}\left(\sum_{j \in \partial i\text{ and Byz}} \xi_j C_{ij} > \frac{kd_is}{2} \right) \label{eqn:split}
                \end{align}
                We need to show two things
                \begin{enumerate}[label=(\arabic*)]
                    \item The number of good voters that will get eliminated are at most a constant fraction. \label{case1}
                    \item The Byzantine voters that are not eliminated in any round can not cause the second term to get very large. \label{case2}
                \end{enumerate}
                
                The first term in Equation \ref{eqn:split} will work out the same way as originally upto constant factors because we will show that only a constant factor of the good voters can be removed with high probability.
                
                \textbf{Proof of \ref{case1}}: 
                
                We will now show that Line 13 in Algorithm \ref{alg:byzalg1} does not eliminate more than a constant fraction of the good voters. This is necessary to show a bound on the good voters' term in Equation \ref{eqn:split}.

                We will now calculate a bound on the maximum number of good out-of-bound($> 5\delta$ away from the $\hat{m}$) entries that can be there across all $\xi$. We claim that this number is $2k^{1-Q} + {2k^{1-\alpha}}$. This is because by a similar argument as claim \ref{clm1} we know that by the Chernoff-Hoeffding bound that the probability with which there will be so many out-of-bound entries is:
                \begin{align*}
                    \mathbb{P}&\left(\frac{1}{k} \sum Y_i \geq \frac{2}{k^Q} + \frac{2}{k^\alpha} \right)\\
                    &\leq
                    \exp\left(\infdiv*{\frac{2}{k^Q} + \frac{2}{k^\alpha}}{\frac{2}{k^Q}} k\right)\\
                    &\leq \exp\left(-\infdiv*{p+p'}{p}k\right) \tag{$p' = \frac{2}{k^\alpha}$ and $p = \frac{2}{k^Q}$}\\
                    &\leq \exp\left( \left(-(p + p') \cdot \log \left(1 + \frac{p'}{p}\right) - (1-p-p') \log\left(\frac{1-p-p'}{1-p} \right) \right) k \right)\\
                     &\leq \exp\left( \left(-(p + p') \cdot \log \left(1 + \frac{p'}{p} \right) - \log\left(1-p-p'\right) \right) k \right)\\
                      &\leq \exp\left( -(p + p') \cdot \left(\log \left(1 + \frac{p'}{p} \right) - \frac{3}{2} \right) k \right) \tag{$\log(1-t) \leq 3t/2$ for $t <0.5$}\\
                      &\leq \exp\left( -p'k/2 \right) = \exp(-k^{1-\alpha})
                \end{align*}
                For the last equation to hold we need $p'/ p \geq e^2 - 1$. In other words, we need $k^Q \geq k^\alpha \cdot (e^2 - 1)$, we can see that this can clearly be achieved for all values of $\alpha \in [0,1]$ for large values of $k$ if $Q > 1$. However, the probability that we have computed above is for some $\xi$. For it to be applicable over all $\xi$ we need $2^{d_i} \cdot \exp(-k^{1-\alpha})$ to be very small, which we can see is satisfied when $k^{1-\alpha} \geq d_i$. Since we have defined $\alpha = 1 - \log d_{max} / \log k$, we can see that it is always true.
                
                We see that for a voter to be removed there need to be $8k^{1-Q} + {8k^{1-\alpha}}$ out-of-bounds entries, therefore when we remove a single voter we will subsequently be removing 3 Byzantine voters. This ensures that at most $k/6$ good voters can be removed. Therefore in total after the removal, the number of good voters are $\geq \frac{k}{3}$. Therefore using the same equation from \cite{negahban2017rank}, with probability $\geq 1 - 2^{d_i} \cdot \exp(-k^{1-\alpha})$ we can conclude that:
                \begin{align}
                    \sum_{j\in \partial i} \sum_{\xi_j \in \{-1, 1\}} \mathbb{P}\left(\sum_{j\in good} |C_{ij}| > \frac{kd_is}{2} \right) \leq \exp \left(-  \frac{k}{3} d_{max} s^2 + d_i \log 2 \right)  \label{good_bound}
                \end{align}
                The above result comes directly by applying Hoeffding's inequality and is similar to the result in \cite{negahban2017rank} as we are only handling good agents here. We can see that for this result to hold with high probability we need $\frac{ks^2}{3} \geq \log 2$ since $d_i \in \Omega(\log n)$. Therefore we come to the conclusion of:
                \begin{equation}
                    \label{eqn:sbound1}
                    s \geq \sqrt{\frac{3 \log 2}{k}}
                \end{equation}

                \textbf{Proof of \ref{case2}}: We will now prove that the second term also does not occur with high probability. We will do this in two parts. We will consider voters that lie close to the mean and the voters that do not lie close to the mean.
                
                We initially only consider the cases where the voters lie within the $5\delta$ range from $\hat{m}$. We can see that for each $\xi$, the Byzantine entries of $D_{i,j,v}$ can only be at a distance of $9\delta$ away from the mean or they cannot fall in the $5\delta$ range. Therefore $\sum_j \xi_j C_{i,j}$ can only be at most $9\delta f $ away from the mean. Therefore all we need to show is that  $\frac{9 \delta k}{2}  \leq \frac{1}{4} k d_i s$, since $f \leq k/2$. This is equivalent to
                \begin{equation}
                    \label{eqn:sbound2}
                    s \geq 9 \sqrt{\frac{2Q\log k}{d_i}}.    
                \end{equation}
                Now all we need to do is consider the cases where the voters do not lie within the $\delta$ range. We can see that we eliminate voters if for a particular $\xi$ they have a lot of out-of-bound entries. Therefore we can claim that for any $\xi$ there are only $8 k^{1-Q} + {8k^{1-\alpha}}$ entries that are not in the $5\delta$ range and each can at most be $d$, we can see that the maximum deviation will at most be $8k \cdot \left(\frac{1}{k^Q} + \frac{1}{k^\alpha} \right) d_i$. Since we know that $k^Q \geq k^\alpha \cdot(e^2 - 1)$ we can conclude that the max deviation will be $\frac{8e^2k^{1-\alpha}}{e^2-1} \leq \frac{1}{4} kd_is$. Using an approximate value for $e$ we get:
                \begin{equation}
                    \label{eqn:sbound3}
                    s \geq \frac{38}{k^\alpha}.
                \end{equation}
                Therefore we can conclude that the second term does not occur. Therefore we can conclude that:
                \begin{align}
                    \sum_{j\in \partial i} \sum_{\xi_j \in \{-1, 1\}} \mathbb{P}\left(\sum_{j\in Byz} |C_{ij}| > \frac{kd_is}{2} \right) &\leq   \exp\left(-\frac{k\epsilon}{18} + d_i \log 2 \right) +  \exp\left(-\frac{k\epsilon^2}{2} \right)  \label{bad_bound}
                \end{align}
                The last equation comes by using Equation \ref{eqn:epsrel}. Combining Equations \ref{eqn:sbound1}, \ref{eqn:sbound2}, \ref{eqn:sbound3} we get
                
                    $$s = \max \left(\sqrt{\frac{3 \log 2}{k}}, \frac{38}{k^\alpha},  9\sqrt{\frac{2Q\log k}{d_{min}}}\right)$$
                
                However, since $d_{min} < k$ we can get rid of the first term to get:
                
                \begin{equation}
                    \label{eqn:sval}
                    s = \max \left(\frac{38d_{max}}{k}, 9 \sqrt{\frac{2Q\log k}{d_{min}}}\right)    
                \end{equation}

               Combining equations \ref{good_bound} and \ref{bad_bound} with equation \ref{eqn:split} we get:
                \begin{align*}
                 \mathbb{P}(R_i \geq s) &\leq e^{\left(-\frac{k\epsilon}{18} + d_i \log 2 \right)} +  e^{-\frac{k\epsilon^2}{2}} + \exp\left (- \frac{k}{3} d_i s^2 + d_i \log 2 \right)
                 + \exp(-k^{1-\alpha}+d_i\log 2)\\
                 &\leq   e^{\left(-\frac{k\epsilon}{18} + d_{max} \log 2 \right)} +  e^{-\frac{k\epsilon^2}{2}} + \exp \left (- \frac{k}{3} d_{max} s^2 + d_{max}\log 2 \right)
                 + \exp(-k^{1-\alpha}+d_{max}\log 2)
                \end{align*}
                Since this is a sum of 4 terms we individually analyse them. Firstly we see that the first two terms are clearly small because $k \in \Omega(\log n)$. For the third term, we see that because of equation \ref{eqn:sval}, $ks^2/3 \geq 1 + \log 2$ therefore making the third term fairly small as well. Finally, the definition of $\alpha$ ensures that the last term is small. Since we have to find a value for $k$ we can set $k$ to be:
                \begin{equation}
                    \label{eqn:kval}
                    k = \max\left(\frac{18}{\epsilon}(d_{max}\log 2 + c^\ast\log n), \frac{2c^\ast\log n}{\epsilon^2}, \frac{3(1 + \log 2)}{s^2}\right) 
                \end{equation}
                We see that the third term is asymptotically smaller than the first two, we can remove it from consideration. By setting 
                \begin{equation*}
                    k \geq 18 \frac{d_{max}}{\epsilon^2}
                \end{equation*}
                we get 
                \begin{equation*}
                    \mathbb{P}(R_i \geq s) \leq 1 - \frac{4}{e^{d_{max}(1 - \log 2)}} 
                \end{equation*}
                The \RC analysis later takes into account the probability that all degrees are in between $np/2$ and $3np/2$. Therefore we get
                \begin{equation*}
                    \mathbb{P}(R_i \geq s) \leq 1 - \frac{4}{n^{5C^2(1 - \log 2)}} 
                \end{equation*}
                Applying the Union bound we know that 
                \begin{align}
                    \mathbb{P}\left( \lVert \bar{\Delta} \rVert_2 \geq s\right)
                    &\leq 2n \mathbb{P}(R_i \geq s)
                    \label{eqn:deltabound}
                \end{align}
                Therefore we can say that:
                \begin{equation*}
                    \lVert \bar{\Delta} \rVert_2 \leq \max \left(\frac{38d_{max}}{k}, 9 \sqrt{\frac{2Q\log k}{d_{min}}}\right)  
                \end{equation*}
                 with probability $\geq 1 - \frac{8n}{n^{5C^2(1 - \log 2)}} \geq 1 - \frac{8}{n^{1.5C^2-1}}$. Since $\max \left( 38, 9\sqrt{2Q} \right) \leq 40$ for values of $1 \leq Q \leq 9$, we can say 
                \begin{equation*}
                    \lVert \bar{\Delta} \rVert_2 \leq 40 \max \left( \frac{d_{max}}{k} , \sqrt{\frac{\log k}{d_{min}}} \right)
                \end{equation*}
                with probability $\geq 1 - \frac{8}{n^{1.5C^2 - 1}}$. Hence Proved.
            \end{proof}
            
            We will now prove that the diagonal entries of $\Delta$ are bounded
            
            \begin{lem}
                Algorithm \ref{alg:byzalg1} removes some voters such that we have:
                \begin{equation*}
                    \lVert D \rVert_2 \leq 40 \max \left( \frac{d_{max}}{k} , \sqrt{\frac{\log k}{d_{min}}} \right)
                \end{equation*}
                with probability $\geq 1 - \frac{8}{n^{1.5C^2-1}}$. 
                \label{lem:byz_diag}
            \end{lem}
            
            \begin{proof}
                From Equations (15) and (16) we know that
                \begin{equation}
                    \Delta_{ii} = - \frac{1}{kd_{max}}  \sum_{j \neq i} C_{ij}
                \end{equation}
                
                Therefore we can write
                \begin{align*}
                    \mathbb{P}\left( \lVert D \rVert_2 \geq 40 \max \left( \frac{d_{max}}{k} , \sqrt{\frac{\log k}{d_{min}}} \right)\right) &\leq \sum_{i=1}^n \mathbb{P}\left( |\sum_{j \neq i} C_{ij}| \geq 40 \max \left( \frac{d_{max}}{k} , \sqrt{\frac{\log k}{d_{min}}} \right) \right)
                \end{align*}
                Since in the previous part we had proved a much strong bound wherein we had shown that $\sum_{i=1}^n \mathbb{P}\left( \sum_{j \neq i} |C_{ij}| \geq  40 \sqrt{\frac{\log d_i}{d_i}} \right) \leq \frac{8}{n^{1.5C^2 - 1}}$, we can conclude that by the triangle inequality we have:
                \begin{align*}
                    \mathbb{P}\left( \lVert D \rVert_2 \geq 40 \max \left( \frac{d_{max}}{k} , \sqrt{\frac{\log k}{d_{min}}} \right) \right) &\leq \frac{8}{n^{1.5C^2 - 1}}
                \end{align*}
                Hence Proved.
            \end{proof}
            
            Now combining Lemmas \ref{lem:byz_non_diag} and \ref{lem:byz_diag}  we have proved Lemma \ref{thm:byz1}.
            \end{proof}
            }

            \ifthenelse{\boolean{if_appendix_algo}}{}{\PROOFALG} 
           
        \subsection{Fast Byzantine Spectral Ranking}
                
        While Algorithm \ref{alg:byzalg1} performs asymptotically similar to the \RC algorithm, it is fairly slow requiring at least $\Omega(2^{d_{min}} \cdot n)$ time. We now present the Fast Byzantine Spectral Ranking~(FBSR) Algorithm for more practical settings. Essentially in our previous algorithm, we ensure that  $\mathbb{P}( \sum_{j\in \partial i} |C_{ij}| > kd_is )$ is upper-bounded. This takes up a lot of time because essentially we iterate over all $\xi \in \{-1, 1\}^{d_i}$ in the algorithm. But, this can be improved. If $d_i = \beta \log_2 n$ we can bucket the entries into $\beta$ buckets~($[B_1, \dots B_\beta]$) of size $\log_2 n$ each. Therefore an equivalent formulation would be to upper-bound $\sum_{l=1}^{\beta} \mathbb{P}( \sum_{j\in \partial i \text{ \& } j \in B_l} |C_{ij}| > \frac{kd_is}{\beta} )$. Here instead of summing up $d_i$ absolute values, we will only be summing up $\log_2 n$ values $\beta$ times.

        \begin{algorithm}
            \caption{Fast Byzantine Spectral Ranking}
            \begin{algorithmic}[1]
                \State $\alpha \gets 1 - \log((2+C/8) \log n) / \log k $
                \For {each object $i$ to be compared}
                    \State $\mathsf{max\_size} \gets \log_2 n$ \Comment{Can be made larger than $\log_2 n$}
                    \State $\beta \gets \left\lceil \frac{d_i}{\mathsf{max\_size}} \right\rceil$
                    \State Split $N_G(i)$ into $\beta$ buckets such that they are equally sized and let them be $[B_1, \dots, B_{\beta}]$
                    \For{$\mathsf{bucket} \gets 1$ to $\beta$}
                        \State $\mathsf{votes}  \gets \mathsf{Bound\_Sum\_Deviations}$($i$, $B_{\mathsf{bucket}}$, $\mathsf{size}(B_{\mathsf{bucket}})$, $k$, $\alpha$)
                        \State Using $\mathsf{votes}$ update $P$ as described in Equation \ref{eqn:pfromadj} 
                    \EndFor    
                \EndFor
                \State Compute the stationary distribution $\pi$ which is the limit of $p^T_{t+1} = p^T_{t} P$.
            \end{algorithmic}
            \label{alg:byzalg2}
        \end{algorithm}   
        
        \begin{thm}
        \label{thm:fbsr}
            Given $n$ objects, let the comparison graph be $G(n,p)$. Each pair in $G$ is compared $k$ times with the outcomes of comparisons produced as per the Byzantine-BTL model with weights $[w_1 , \dots , w_n]$. Then there exists an algorithm such that when $F \leq K(1-\epsilon)/2$, $p = 10C^2 \log n / n$, $\epsilon > 0$, and $k \geq 18(2+C/8)\log n/\epsilon^2$ the following bound on the error 
                $$\frac{\lVert \pi - \tilde{\pi} \rVert}{\lVert \tilde{\pi} \rVert} \leq 480 b^{5/2} \max \left( {\frac{\log n}{k}}, \sqrt{\frac{\log \log n}{\log n}} \right)$$
            holds with probability $\geq 1 - (6+240C^2)n^{-C/8}$ that runs in $\mathcal{O}(n^2)$ time.
        \end{thm}
        
        The detailed analysis can be found in Appendix \ref{sec:fbsralgana}. 
        % \textbf{Remark.} The $k \in \omega(\log n)$ might seem restrictive, but because of our proposed bound $\log n / k$ the RHS of Theorem \ref{thm:fbsr} will be a constant for $k \in \mathcal{O}(\log n)$.
        
        \subsection{An Impossibility Result and optimality of FBSR}
        \label{sec:imp}
        
        The main idea in the BSR algorithm was the estimation of an accurate mean for all values of $\xi$. However, we can see that finding the mean can become very challenging if the Byzantine voters outnumber the good voters. This is because the Byzantine voters can always create a shifted binomial distribution and trick us into believing that an entirely different value is the mean.
        
        We answer a natural question: is it even possible to find algorithms when $F$ crosses $K/2$? We prove a fundamental result showing that for any ranking algorithm to give a satisfactory result, a majority of good voters will always be required. Formally we show that:
        
        \begin{thm}
            If $F \geq K/2$, then no algorithm can for all weights ($\tilde{\pi}$), output weights ($\pi^\ast$) such that $$\frac{\lVert {\pi}^\ast - \tilde{\pi} \rVert}{\lVert \tilde{\pi} \rVert} \leq f(n)$$ with probability $> 1/2$, where $f(n)$ is a function that converges to 0 as $n$ goes to $\infty$. 
        \end{thm}
        
        \begin{proof}[Proof Sketch]
            
            We prove this result by contradiction, let us suppose there is an algorithm $\mathcal{A}$ which can give a satisfactory ranking for all weights. We consider the $F = K/2$ case, as the case where $F > K/2$ can clearly be reduced to the $F= K/2$ case.
            
            We then consider two instances:
            \begin{enumerate*}[label=(\arabic*)]
                \item Good voters vote according to $\tilde{\pi}$ and the Byzantine voters vote according to $\tilde{\pi}'$.
                \item Byzantine voters vote according to $\tilde{\pi}$ and the good voters vote according to $\tilde{\pi}'$.
            \end{enumerate*}
            We go on to claim that when $F = K/2$ these instances are indistinguishable to any algorithm. Since $\mathcal{A}$ succeeds with probability $>1/2$ for any instance, we can say by using the union bound that there must be a $\pi^\ast$ such that it is close to both $\tilde{\pi}$ and $\tilde{\pi}'$. We can then choose the values of $\tilde{\pi}$ and $\tilde{\pi}'$ such that they are far from each other and therefore showing that there is no $\pi^\ast$ that it is close to both $\tilde{\pi}$ and $\tilde{\pi}'$. Giving us a contradiction. \end{proof}
            
            The detailed proof can be found in Appendix \ref{sec:impossproof}.

        \def \PROOFIMPOSSIBLE {
        \begin{proof}
            We will prove by contradiction. Let there be an algorithm $\mathcal{A}$ such that for all input weights it returns output weights such that $ \frac{\lVert \pi^\ast - \tilde{\pi} \rVert}{\lVert \tilde{\pi} \rVert} \leq f(n)$ with probability $> 1/2$, where $f(n)$ converges to $0$. We consider the case where $F = K/2$. We can see that the cases where $F > K/2$ can be handled by making $2F - K$ Byzantine voters just output $1$ irrespective of the queries. We can assume that the algorithm is aware of these $2F-K$ voters and will therefore always ignore these voters. Therefore the problem reduces to the $F = K/2$ case.  
            
            Consider an instance~(Instance 1) where $\tilde{\pi}$ is defined as follows:
            \begin{equation}
                \tilde{\pi}_i = \begin{cases}
                    \frac{1}{n + (b - 1) \cdot \lceil \frac{n}{2} \rceil } & \text{for }1 \leq i \leq \lfloor \frac{n}{2} \rfloor\\
                    \frac{b}{n + (b - 1) \cdot \lceil \frac{n}{2} \rceil } & \text{for }\lfloor \frac{n}{2} \rfloor < i \leq n 
                \end{cases}
                \label{eqn:pival}
            \end{equation}
            Let us consider the case where the good voters are numbered $[1, \frac{K}{2}]$, while the Byzantine voters are numbered $[\frac{K}{2} + 1, K]$. We can use $\mathcal{A}$ to give us a $\pi^\ast$ such that $ \frac{\lVert \pi^\ast - \tilde{\pi} \rVert}{\lVert \tilde{\pi} \rVert}  \leq f(n)$ with probability $> 1/2$. We now outline the strategy that the Byzantine voters can use. The Byzantine voters give their results according to some $\tilde{\pi}'$. Where:
            \begin{equation}
                \tilde{\pi}'_i = \begin{cases}
                    \frac{b}{n + (b - 1) \cdot \lceil \frac{n}{2} \rceil  } & \text{for }1 \leq i \leq \lceil \frac{n}{2} \rceil\\
                    \frac{1}{n + (b - 1) \cdot \lceil \frac{n}{2} \rceil } & \text{for }\lceil \frac{n}{2} \rceil < i \leq n 
                \end{cases}
                \label{eqn:pival1}
            \end{equation}

            In another instance~(Instance 2), we consider the good and Byzantine voters flip their numberings i.e. the Byzantine voters are numbered from $[1, \frac{K}{2}]$ and the good voters are numbered from $[\frac{K}{2} + 1, K]$. Furthermore, let the input weights be $\tilde{\pi}'$ and the Byzantine voters vote according to $\tilde{\pi}$. Based on the description of our instance and the existence of algorithm $\mathcal{A}$, with probability $> 1/2$ we will get a weight output from $\mathcal{A}$ such that
            \begin{equation}
                \label{eqn:trick1}
                \frac{\lVert \pi^\ast - \tilde{\pi}' \rVert}{\lVert \tilde{\pi}' \rVert} \leq f(n)    
            \end{equation}
            But to any algorithm, the above two instances are identical and therefore with probability $> 1/2$:
            \begin{equation}
                \label{eqn:trick2}
                \frac{\lVert \pi^\ast - \tilde{\pi} \rVert}{\lVert \tilde{\pi} \rVert}  \leq f(n)
            \end{equation}

            We can therefore conclude that with a non-zero probability we require both Equations \ref{eqn:trick1} and \ref{eqn:trick2} by a simple application of the union bound. We see that:
            \begin{align*}
                \lVert \tilde{\pi} - \tilde{\pi}' \rVert &\leq \lVert \tilde{\pi} - \pi^\ast \rVert + \lVert {\pi}^\ast - \tilde{\pi}'\rVert \tag{Triangle Inequality}\\
                \lVert \tilde{\pi} - \tilde{\pi}' \rVert &\leq f(n) \cdot \left( \lVert \tilde{\pi} \rVert 
                + \lVert \tilde{\pi}' \rVert\right) \tag{Equations \ref{eqn:trick1} and \ref{eqn:trick2}}\\
                \frac{1}{2} \cdot \frac{(b-1)\cdot \sqrt{n-1}}{\sqrt{b^2 \cdot \lceil{n/2}\rceil + \lfloor n/2\rfloor}} &\leq f(n) \tag{Equations \ref{eqn:pival} and \ref{eqn:pival1}}
            \end{align*}
            For larger $n$, the LHS converges to $\geq \frac{b-1}{2b}$, while the RHS converges to 0. Giving a contradiction.
        \end{proof}
        }
        \ifthenelse{\boolean{if_appendix_imposs}}{}{\PROOFIMPOSSIBLE} 
        
        \textbf{Remark.} This impossibility consequently shows us that the FBSR algorithm~(cf. Theorem \ref{thm:fbsr}) is optimal in terms of tolerance towards the Byzantine fraction~($F/K$). 
    
        %The impossibility result has been shown here for the $2$-norm but the same proof will work out for any other $p$-norm as well. 
        \subsection{Comparison with \cite{pmlr-v119-agarwal20a}}
        \label{sec:agarwal}
        
        While the setting of \cite{pmlr-v119-agarwal20a} is different from our setting as mentioned in Section \ref{sec:relworks}, if we were to consider both works in  Byzantine voter setting a comparison could be made by considering an edge that has a Byzantine voter to be corrupted as per their model. Here, their algorithm ,to have a valid convergence proof, needs $k \in \Omega(\log n)$. This would imply that $F / K \in O( \log \log n / \log^2 n)$ for the sparse graphs and $F /K \in O( 1 /\log n)$ for denser graphs. Our algorithms can potentially tolerate a significantly higher corruption rate of $1/2$. 
  
        Another way to compare our results with \cite{pmlr-v119-agarwal20a} would be to set $k = \Omega(\log n), p = \Theta(\log n / n) $ for our model and $k' = 1, p' = pk / k' $ for their model. This allocation ensures only a single voter per edge for their model and equal number of comparisons in expectation. We see that their algorithm does not converge ($L_1 \in \mathcal{O}(\sqrt{\log n})$), while our algorithms converges $\left( L_2 \in \mathcal{O} \left(\max\left(\frac{\log n}{k}, \sqrt{\frac{\log k}{\log n}} \right) \right) \right)$. Furthermore, our algorithm can handle a corruption rate of $1/2$ while their algorithm can handle a corruption rate of $\mathcal{O} \left( \log \log n / \log n\right)$\footnote{they also come up with an algorithm for denser graphs that can handle a corruption rate of $1/4$, however the algorithm has worse time-complexity and requires $p \in n^{\epsilon - 1}$ for $\epsilon > 0$}.        
    \section{Experimental Results}
    \label{sec:expresults}
    
    In this section, we confirm our theoretical bounds with experiments conducted on synthetic and real data. While truly Byzantine strategies might be hard to find and can even be computationally infeasible, we consider some practically possible strategies. We show the performance of \RC against these strategies, supporting our results in section \ref{sec:fail}. We then proceed to show that these strategies are not as effective against the BSR/FBSR algorithm which gives better rankings. 
    
        \textbf{Experimental Settings:}
        We see that for the FBSR Algorithm to work well we need the existence of at least a few entries that are at a distance $\geq 5\delta $ from $\hat{m}$. However, since we are summing up $\log n$ entries we will need $5\delta \leq \log n$. Otherwise, the BSR algorithm will run exactly as the \RC algorithm. Using the value of $\delta$ from Algorithm \ref{alg:byzalg2} and setting $Q = 1$ and using $k \geq \log n$ we get:
        $$\frac{25}{2} \log \log n \leq \log n $$
        While it is asymptotically true, we see that the RHS only becomes greater than the LHS when $n \geq 1.18 \times 10^{21}$. However, as our proof was designed to work against any Byzantine strategy, and at the same time the convergence derivation of \cite{negahban2017rank} was fairly loose (for example, the union bound of $2^d$ probabilities to ensure the sum of absolute terms is bounded) the FBSR algorithm empirically works for a variety of strategies even for significantly smaller thresholds. We applied the FBSR Algorithm with the following modifications considering the smaller $k$ and $n$ values: \begin{enumerate*}[label=(\arabic*)]
            \item setting $5\delta = 1+\sqrt{\mathsf{size}(B_{\mathsf{bucket}})}$ and
            \item setting $\mathsf{max\_out}$ as $k/20$.
            \label{modif2}
        \end{enumerate*} All experimental results in sections \ref{sec:syndata} and \ref{sec:realdata} have been taken by taking the mean over 10 random trials.
    
        \textbf{Synthetic Data:}
        \label{sec:syndata}
        We consider the following parameters in our testing \begin{enumerate*}[label=(\arabic*)]
            \item $n = 200$ \item $k = 100$ \item $p = 20\log n / n$ and \item $w_i = \text{Uniform}(1, 100)$. These are then normalized such that $\sum \tilde{\pi}_i = 1$.
        \end{enumerate*} We compare our model with the \RC algorithm on synthetic data. We consider four potential strategies used by Byzantine voters:
        
        \begin{enumerate}[label=(\arabic*)]
            \item \emph{Fixed Order Vote:} The Byzantine voters vote according to some random pre-determined order. This is a fairly simple strategy that does not even require the Byzantine voters to know the $\tilde{\pi}$ values.
            \item \emph{Opposite Vote:} The Byzantine voters vote exactly the reverse of the weights. We had shown in section \ref{sec:fail} that this strategy leads to a failure for some weights.
            \item \emph{Opposite Vote Probabilistic:} The Byzantine voters are equivalent to good voters, except their final answer is flipped. For example, the Byzantine voters will vote for object $j$ with probability $\frac{w_i}{w_j+w_i}$ and vote for object $i$ with probability $\frac{w_j}{w_j+w_i}$.
            \item \emph{Random Subset:} For each pair, the Byzantine voters either decide to behave like good voters with half probability or decide to {\emph opposite vote}.
        \end{enumerate}

        Figures \ref{fig:l2_combined} and \ref{fig:tau_combined} demonstrate the performance of the \RC algorithm and the FBSR algorithm against the aforementioned strategies. For Figure \ref{fig:l2_combined}, we plot the relative $L_2$-Error against the Byzantine Fraction~($F/K$). Similar to our results in section \ref{sec:fail}, we see that for most strategies the relative $L_2$-Error for the \RC Algorithm grows linearly with $F/K$. We can see that the relative $L_2$-Error for both Fixed Order Vote and Opposite Vote strategies are fairly high for the \RC algorithm and that the Fixed Order Vote dominates the relative $L_2$-Error for the FBSR algorithm. However, even though the Fixed Order Vote is the best strategy against the FBSR algorithm. The FBSR algorithm is still able to perform around two times better than the \RC algorithm across all values of $F/K$.
        
        We also confirm that our rankings are close to the actual rankings. To do this we consider Kendall's Tau Correlation~\citep{kt} between our ranking and the actual ranking. From Figure \ref{fig:tau_combined}, we see that for the \RC algorithm, the Opposite Vote strategy is able to give very low values of $\tau$. When the Byzantine fraction is around $0.3$ we get a Kendall's Tau Correlation close to $0$~(essentially a random ranking). We also see that the FBSR algorithm is able to give better results for all strategies.

        \begin{figure}
            \begin{minipage}[b]{0.33\textwidth}
                \centering
                \includegraphics[width=\textwidth]{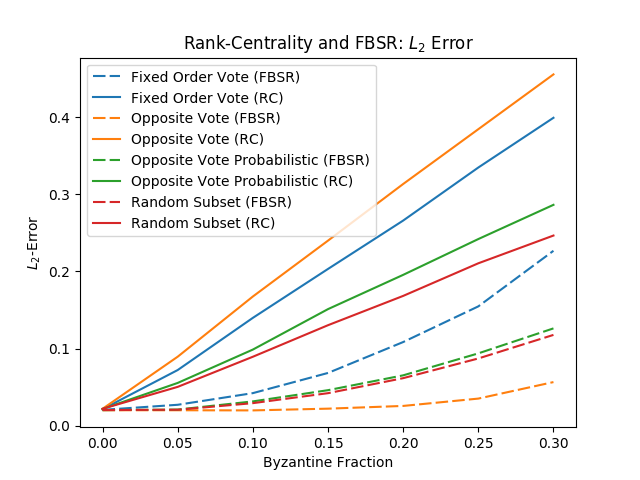}
                \caption{$L_2$ vs $F/K$}
                \label{fig:l2_combined}
            \end{minipage}
            \begin{minipage}[b]{0.33\textwidth}
                \centering
                \includegraphics[width=\textwidth]{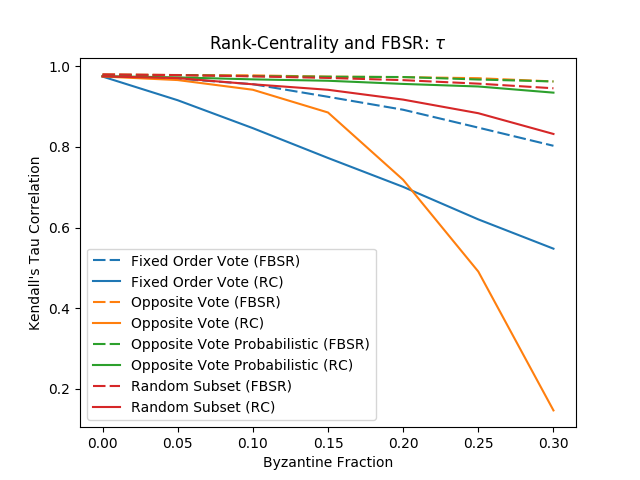}
                \caption{$\tau$ vs $F/K$}
                \label{fig:tau_combined}
            \end{minipage}
            \begin{minipage}[b]{0.33\linewidth}
                \centering
                \includegraphics[width=\textwidth]{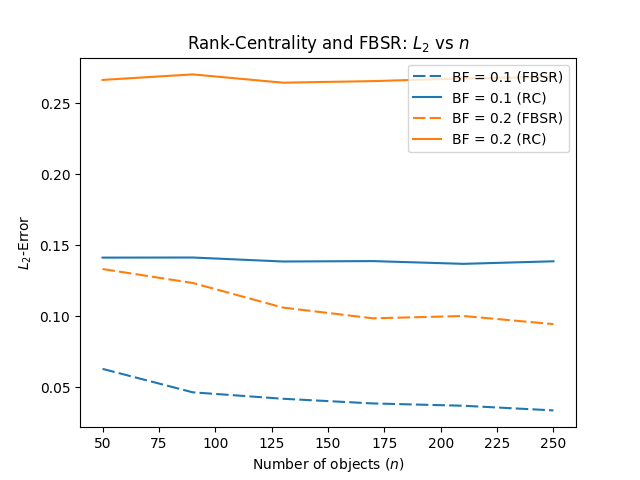}
                \caption{$L_2$ vs $n$}
                \label{fig:rcperf_item_l2}
            \end{minipage}
        \end{figure}

        To further show the convergence of the algorithms in the Byzantine-BTL model for larger $n$, we plot the relative $L_2$ error and Kendall's Tau Correlation against the number of objects in Figures \ref{fig:rcperf_item_l2} and \ref{fig:rcperf_item_tau} for the Fixed Order Vote strategy. We consider values of $F/K$ in $[0.1, 0.2]$. These results come when we have set $k = n$. We see that for a large range of the number of objects both metrics have remained constant for the \RC algorithm thus supporting our hypothesis. On the other hand, we see that as we increase the number of objects the $L_2$-error gradually falls and the Kendall's Tau Correlation gradually increases. The rather slow descent of the relative $L_2$-error is backed by Theorem \ref{thm:fbsr} where we see that for larger $k$ the relative $L_2$ error is $\in \mathcal{O}\left(\sqrt{\frac{\log \log n}{\log n}} \right)$.

        \textbf{Real Data}:
        \label{sec:realdata}
            Experimentation with real data was not straightforward, considering that most datasets do not have voter labels nor do they have voter queries as requested by the BSR algorithm~(one voter is always asked multiple queries of the form $(i,j)$ where $i$ is fixed and $j$ is varied). To get around this issue, we use complete rankings. We consider the Sushi dataset comprising $5000$ voters ordering the $10$ sushis from most preferred to least preferred. Given the preference order, we get $45$ pairwise votes from each of the voters. Since we are assuming that we are dealing with permutations here, we assume the Byzantine voters can use the following strategies:
            \begin{enumerate}[label=(\arabic*)]
                \item \textit{Fixed Order Vote / Opposite Vote:} Same as previously defined, but this time the entire permutation is given.
                \item \textit{Opposite Random Flips:} Each Byzantine voter starts with the opposite permutation and then swaps some of the objects at random.
            \end{enumerate}
            
            Since the Sushi dataset has many voters we do not use our \ref{modif2}nd modification mentioned above. We compute the true weights by applying the \RC algorithm directly~(thus the \RC algorithm has a big advantage, especially for smaller weights). Figures \ref{fig:rc_perf_item_l2_sushi} and \ref{fig:tau_perf_item_l2_sushi} show the \RC algorithm and the BSR algorithm's performance. We see that the BSR algorithm achieves lower relative $L_2$-errors/Kendall's Tau Correlation for most Byzantine fractions and strategies.  The comparatively small improvement in Kendall's Tau Correlation is attributed to a rather small difference in the weights of the various Sushi~(5 sushis have weights $\in [0.085, 0.13]$).
        
        \begin{figure}
            
            \begin{minipage}[b]{0.33\linewidth}
                \centering
                \includegraphics[width=\textwidth]{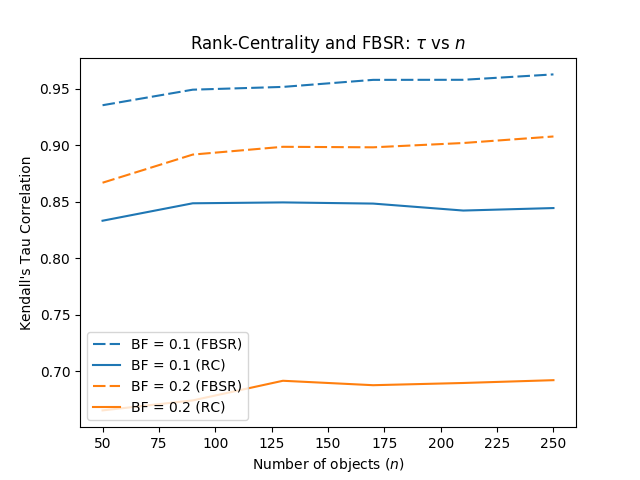}
                \caption{$\tau$ vs $n$}
                \label{fig:rcperf_item_tau}
            \end{minipage}
            \begin{minipage}[b]{0.33\linewidth}
                \centering
                \includegraphics[width=\textwidth]{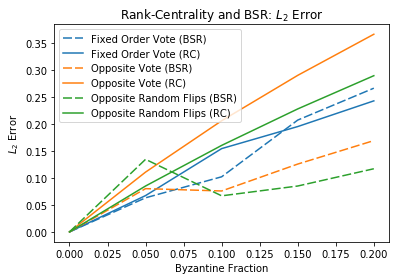}
                \caption{Sushi: $L_2$}
                \label{fig:rc_perf_item_l2_sushi}
            \end{minipage}
            \begin{minipage}[b]{0.33\linewidth}
                \centering
                \includegraphics[width=\textwidth]{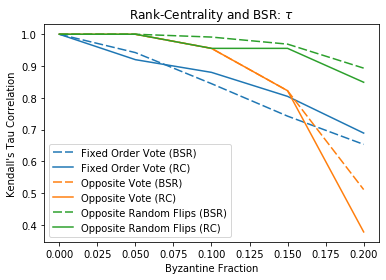}
                \caption{Sushi: $\tau$}
                \label{fig:tau_perf_item_l2_sushi}
            \end{minipage}
        \end{figure}

    \section{Conclusion and Further Work}
    \label{sec:conclusion}
	
	We studied the potential disruption that Byzantine voters can do to the widely-used \RC algorithm and proposed fast solutions to counter the same. Our algorithms and analysis techniques are quite non-trivial and indicate to us the need for developing this theory further. We require $d \in \Theta(\log n)$ and $k \in \Omega(\log n)$ for the success of the BSR and FBSR algorithms. However, \cite{negahban2017rank} only required $d \in \Omega(\log n)$ and $k \in \Omega(1)$. An interesting direction would be to explore whether sub-logarithmic values of $k$ will be possible. Alternatively, a lower bound on the number of required votes per edge for an algorithm to succeed with constant probability would also be an interesting direction.
	%Alternatively, general graphs with larger degrees instead of the \ER graphs could be considered.
	
	%Another downside of our work is that all of our queries to one voter are of the form $(i,j)$ where $i$ is fixed and $j$ is kept variable. While this may be convenient for the voters~(they just need to choose a set of objects that are better than $i$ amongst a given set) it might even be cumbersome for some voters. The natural question is whether we can develop flexible algorithms which can give an accurate ranking despite the nature of the queries asked to the voters.
	%Another interesting direction would be to  consider the problem of Byzantine Spectral Top-$K$ ranking. Especially since the first step of some of these algorithms~\citep{spectralmle} is to use the \RC algorithm.
	
    % \clearpage

% 	\clearpage
	
	%%%%%%%%%%%%%%%%%%%%%%%%%%%%%%%%%%%%%%%%%%%%%%%%%%%%%%%%%%%%
    \section*{Updated Checklist}

    \begin{enumerate}

    \item For all authors...
    \begin{enumerate}
      \item Do the main claims made in the abstract and introduction accurately reflect the paper's contributions and scope?
        \answerYes{}
      \item Did you describe the limitations of your work?
        \answerYes{}
      \item Did you discuss any potential negative societal impacts of your work?
        \answerYes{}
      \item Have you read the ethics review guidelines and ensured that your paper conforms to them?
        \answerYes{}
    \end{enumerate}

    \item If you are including theoretical results...
    \begin{enumerate}
      \item Did you state the full set of assumptions of all theoretical results?
        \answerYes{}
            \item Did you include complete proofs of all theoretical results?
        \answerYes{}
    \end{enumerate}

    \item If you ran experiments...
    \begin{enumerate}
        \item Did you include the code, data, and instructions needed to reproduce the main experimental results (either in the supplemental material or as a URL)?
        \answerYes{}
        \item Did you specify all the training details (e.g., data splits, hyperparameters, how they were chosen)?
        \answerYes{}
        \item Did you report error bars (e.g., with respect to the random seed after running experiments multiple times)?
        \answerYes{}
        \item Did you include the total amount of compute and the type of resources used (e.g., type of GPUs, internal cluster, or cloud provider)?
        \answerYes{}
    \end{enumerate}

    \item If you are using existing assets (e.g., code, data, models) or curating/releasing new assets...
    \begin{enumerate}
      \item If your work uses existing assets, did you cite the creators?
        \answerNA{}
      \item Did you mention the license of the assets?
        \answerNA{}
      \item Did you include any new assets either in the supplemental material or as a URL?
        \answerNA{}
      \item Did you discuss whether and how consent was obtained from people whose data you're using/curating?
        \answerNA{}
      \item Did you discuss whether the data you are using/curating contains personally identifiable information or offensive content?
        \answerNA{}
    \end{enumerate}

    \item If you used crowdsourcing or conducted research with human subjects...
    \begin{enumerate}
      \item Did you include the full text of instructions given to participants and screenshots, if applicable?
        \answerNA{}
      \item Did you describe any potential participant risks, with links to Institutional Review Board (IRB) approvals, if applicable?
        \answerNA{}
      \item Did you include the estimated hourly wage paid to participants and the total amount spent on participant compensation?
        \answerNA{}
    \end{enumerate}

    \end{enumerate}

    %%%%%%%%%%%%%%%%%%%%%%%%%%%%%%%%%%%%%%%%%%%%%%%%%%%%%%%%%%%%

	\clearpage
	
	\appendix

	\section{Appendix}

	\subsection{A primer on \RC}
	\label{sec:rc_primer}
	
        \subsubsection{\RC Convergence}
        
        \RC proves the below theorem
        
        \begin{thm}[From \cite{negahban2017rank}] \label{rc:thm1}
           Given $n$ objects and a connected comparison graph $G = ([n], E)$, let each pair $(i, j) \in E$ be compared for $k$ times with outcomes produced as per a BTL model with parameters $w_1 , \dots , w_n$. Then, for some positive constant $C \geq 8$ and when $k \geq 4C^2 (b^5\kappa^2/d_{max} \xi^2 ) \log n$, the following bound on the normalized error holds with probability at least $1 - 4n^{-C/8}$ :
           \begin{equation*}
                \frac{\lVert \pi - \tilde{\pi} \rVert}{\lVert \tilde{\pi} \rVert}   \leq \frac{Cb^{5/2}\kappa}{\xi} \sqrt{\frac{\log n}{kd_{max}}}
           \end{equation*}
           where $\tilde{\pi}(i) = w_i / \sum_l w_l$, $b \equiv \max_{i,j} w_i/w_j$ and $\kappa = d_{max} / d_{min}$
           
        \end{thm}
        
        which when extended to \ER random graphs takes the form of
        
        \begin{thm}[From \cite{negahban2017rank}] \label{rc:thm2}
            Given $n$ objects, let the comparison graph $G = ([n], E)$ be generated by selecting each pair $(i, j)$ to be in $E$ with probability $d/n$ independently of everything else. Each such chosen pair of objects is compared $k$ times with the outcomes of comparisons produced as per a BTL model with parameters $w_1 , \dots , w_n$ . Then if $d \geq 10C^2 \log n$ and $kd \geq 128 C^2b^5 \log n$, the following bound on the error rate holds with probability at least $1 - 10n^{-C/8}$:
            \begin{equation*}
                \frac{\lVert \pi - \tilde{\pi} \rVert}{\lVert \tilde{\pi} \rVert}   \leq 8Cb^{5/2} \sqrt{\frac{\log n}{kd_{max}}}
           \end{equation*}
            where $\tilde{\pi}(i) = w_i / \sum_l w_l$ and $b \equiv \max_{i,j} w_i/w_j$
        \end{thm}
        \subsubsection{\RC Convergence Proof Sketch}
        
        The proof from theorem \ref{rc:thm1} to \ref{rc:thm2} is straightforward and has no relation to the fact that there are Byzantine voters. \RC algorithm's convergence proof for Theorem \ref{rc:thm1} requires the proving of three lemmas~(Lemmas 2,3 and 4 given below) 
        
        \begin{lem}[From \cite{negahban2017rank}] \label{rc:lm2}
            For any Markov chain $P = \tilde{P} + \Delta$ with a reversible Markov chain $\tilde{P}$, let $p_t$ be the distribution of the Markov chain $P$ when started with the initial distribution $p_0$. Then, 
            \begin{equation*}
                \frac{\lVert p_t - \tilde{\pi} \rVert}{\tilde{\pi}} \leq \rho^t \frac{\lVert p_0 - \tilde{\pi} \rVert}{\tilde{\pi}} \sqrt{\frac{\tilde{\pi}_{max}}{\tilde{\pi}_{min}}} + \frac{1}{1 - \rho} \lVert \delta \rVert_2 \sqrt{\frac{\tilde{\pi}_{max}}{\tilde{\pi}_{min}}}
            \end{equation*}
            where $\tilde{\pi}$ is the stationary distribution of $\tilde{P}$, $\tilde{\pi}_{min} = \min_{i} \tilde{\pi}(i)$, $\tilde{\pi}_{max} = \max_{i} \tilde{\pi}(i)$ and $\rho = \lambda_{max}(\tilde{P}) + \Delta \sqrt{\frac{\tilde{\pi}_{max}}{\tilde{\pi}_{min}}}$
        \end{lem}
        
        The next lemma bounds the norm of the $\Delta$ matrix:
        
        \begin{lem}[From \cite{negahban2017rank}] \label{rc:lm3}
            For some constant $C \geq 8$, the error matrix $\Delta = P - \tilde{P}$ satisfies
            \begin{equation*}
                \lVert \Delta \rVert_2 \leq C \sqrt{\frac{\log n}{k d_{max}}}
            \end{equation*}
            with probability at least $1 - 4n^{-C/8}$. 
        \end{lem}
        
        The next lemma bounds the value of $1-\rho$:
        
        \begin{lem}[From \cite{negahban2017rank}] \label{rc:lm4}
            If $\lVert \Delta \rVert_2 \leq C \sqrt{\frac{\log n}{k d_{max}}}$ and if the value of $k \geq 4C^2 b^5 d_{max}\log n(1 / d^2_{min} \xi^2 ) $ then,
            \begin{equation*}
                1 - \rho \geq \frac{\xi d_{min}}{b^2d_{max}}
            \end{equation*}
        \end{lem}

        Here we find that Lemmas \ref{rc:lm2} and \ref{rc:lm4} hold irrespective of whether there are Byzantine voters or not. Therefore we need to find a substitute for Lemma \ref{rc:lm3}. The proof of Lemma \ref{rc:lm3} is split into two parts, we mainly focus on the first part where $d$ is sufficiently small and therefore is a sparse graph. Proving Lemma \ref{rc:lm3} for \ER graphs with lower degrees essentially boils down to proving that $R_i = \frac{1}{kd_{max}}\sum_{j \neq i} |A_{ij} - kp_{ij}| \geq s$ with a very low probability, where $s = \frac{C}{2} \sqrt{\frac{\log n + d_{max} \log 2}{k d_{max}}}$. In simpler terms, this implies that the sum of the absolute values of the deviation of the voters' entries from the expected entries for a particular object is not a big number with high probability.

	\ifthenelse{\boolean{if_appendix_failure}}{
	
	    \subsection{\RC Failure}
	    \label{sec:rc_fail_appendix}
	    \PROOFFAILURE
	    \PROOFFAILURELEMMAS

	}{} 
	
	\ifthenelse{\boolean{if_appendix_algo}}{
	
	    \subsection{BSR Algorithm Analysis}
	    \label{sec:alganalysis}
	    
	    We start off by proving Lemma \ref{thm:byz1}.
	    \PROOFALG
	    
	    We now finally prove Theorem \ref{thm:byz2}. Similar to the \RC convergence analysis, using the above results and using Lemmas \ref{rc:lm2} and \ref{rc:lm4} from \cite{negahban2017rank} we can conclude that 
	    
	    $$\frac{\lVert \pi - \tilde{\pi} \rVert}{\lVert \pi \rVert} \leq \frac{80b^{5/2} \kappa}{\xi(G)} \max\left( \frac{d_{max}}{k}, \sqrt{\frac{\log k}{d_{min}}}\right) $$

        where $\kappa = d_{max} / d_{min}$. Similar to \cite{negahban2017rank} we can see that for an \ER graph with probability $\geq 1 - 6n^{-C/8}$ we have $\kappa \leq 3$ and $\xi(G) \geq 1/2$. Substituting we get: 

        $$\frac{\lVert \pi - \tilde{\pi} \rVert}{\lVert \pi \rVert} \leq 480b^{5/2} \max\left( \frac{d_{max}}{k}, \sqrt{\frac{\log k}{d_{min}}}\right) $$

	}{} 
    
    \subsection{FBSR Algorithm Analysis}
    \label{sec:fbsralgana}
    
        Firstly since we are dealing with $p = 10C^2 \log n / n$ and since the proof of the \RC algorithm already takes into account the probability that all degrees will be between $pn/2$ and $3pn/2$ we can say that the minimum degree will at most be $5C^2 \log n$. Therefore we can say that any bucket will have at least $\frac{5C^2\log 2  }{5C^2\log 2+1} \log_2 n$ objects and at most $\log_2 n$ objects. We know that:
        \begin{align}
             \mathbb{P}( \sum_{j\in \partial i} |C_{ij}| > kd_is ) &\leq \sum_{l=1}^{\beta} \mathbb{P} \left( \sum_{j\in \partial i \text{ \& } j \in B_l} |C_{ij}| > ks|B_l|\right)
             \label{eqn:fbsr1}
        \end{align}
        Using the same analysis as above we can say that:
        \begin{align*}
            \mathbb{P}&\left(  \sum_{j\in \partial i \text{ \& } j \in B_l} |C_{ij}| > ks|B_l| \right)  \\ 
            &\leq e^{\left(-\frac{k\epsilon}{18} + |B_l| \log 2 \right)} +  e^{-\frac{k\epsilon^2}{2}} +\exp\left(- \frac{k|B_l|}{3} s^2 + |B_l|\log 2\right) + \exp(-k^{1-\alpha}+ |B_l|\log 2)
        \end{align*}
	    Substituting $k \geq 18(2 + C/8) \log n / \epsilon^2$ we get 
	    \begin{align*}
            \mathbb{P}\left(  \sum_{j\in \partial i \text{ \& } j \in B_l} |C_{ij}| > ks|B_l| \right) \leq \frac{4}{n^{1+C/8}}
        \end{align*}
        We can also see that $\beta$ will be at most $15C^2$. Substituting in Equation \ref{eqn:fbsr1} we get:
        \begin{align*}
            \mathbb{P}(R_i \geq s) \leq \frac{60C^2}{n^{1+C/8}}
        \end{align*}
        where $s$ is defined as
	    \begin{equation*}
            s = \max \left(\frac{38 \log n}{k}, 9 \sqrt{2Q\frac{(5C^2\log 2+1)\log \log n}{5C^2\log n \log 2}}\right)    
        \end{equation*}
        Using $C \geq 1$ and $1 \leq Q \leq 4$ to get:
        \begin{equation*}
            s = 40 \max\left(\frac{\log n}{k}, \sqrt{\frac{\log \log n}{\log n}}\right)    
        \end{equation*}
        Finally using Equation \ref{eqn:deltabound} we get
        \begin{align*}
            \mathbb{P}\left( \lVert \bar{\Delta} \rVert_2 \geq s \right) \leq \frac{120C^2}{n^{C/8}}
        \end{align*}
	    
        The other steps are identical to the analysis of the BSR algorithm

	\ifthenelse{\boolean{if_appendix_algo}}{
	
	    \subsection{Impossibility result proof}
	    \label{sec:impossproof}
	    \PROOFIMPOSSIBLE

	}{} 
	
	\subsection{Broader Impact}
	
	Due to the more theoretical nature of our research, we do not see any potential negative societal impacts associated with the work.

    \subsection{Experimental Results and Code}
    
    We provide the detailed tabular results for Figures \ref{fig:l2_combined}, \ref{fig:tau_combined}, \ref{fig:rcperf_item_l2}, \ref{fig:rcperf_item_tau}, \ref{fig:rc_perf_item_l2_sushi}, and \ref{fig:tau_perf_item_l2_sushi} in Tables \ref{tab:1}, \ref{tab:2}, \ref{tab:3}, \ref{tab:4}, \ref{tab:5} and \ref{tab:6} with error bounds for all the obtained results. The code for producing these can be found \href{https://github.com/Arnhav-Datar/Byzantine_Spectral_Ranking}{here}.
    
    \begin{table}[H]
        \captionsetup{justification=centering}
        \caption{$L_2$-Errors for the FBSR and RC algorithms on Synthetic Data when Byzantine Fraction(BF) is varied. FOV = Fixed Order Vote, OV = Opposite Order Vote, OVP = Opposite Vote Probabilistic, RS = Random Subset}
        \centering
        \begin{adjustbox}{max width=1.15\textwidth,center}
        \begin{tabular}{c|c|ccccccc}
        \toprule
        Strategy & Algo  & 0 & 0.05 & 0.1 & 0.15 & 0.2 & 0.25 & 0.3 \\
        \midrule
        \midrule
        FOV & FBSR &  0.02 $\pm$ 0.001 & 0.03 $\pm$ 0.002 & 0.04 $\pm$ 0.002 & 0.07 $\pm$ 0.004 & 0.11 $\pm$ 0.008 & 0.15 $\pm$ 0.011 & 0.23 $\pm$ 0.015 \\
        & RC &  0.02 $\pm$ 0.002 & 0.07 $\pm$ 0.002 & 0.14 $\pm$ 0.002 & 0.2 $\pm$ 0.005 & 0.27 $\pm$ 0.006 & 0.33 $\pm$ 0.008 & 0.4 $\pm$ 0.01 \\
        \midrule
        OV & FBSR &  0.02 $\pm$ 0.002 & 0.02 $\pm$ 0.001 & 0.02 $\pm$ 0.001 & 0.02 $\pm$ 0.002 & 0.03 $\pm$ 0.002 & 0.04 $\pm$ 0.005 & 0.06 $\pm$ 0.007 \\ 
        & RC & 0.02 $\pm$ 0.001 & 0.09 $\pm$ 0.002 & 0.17 $\pm$ 0.005 & 0.24 $\pm$ 0.004 & 0.31 $\pm$ 0.005 & 0.38 $\pm$ 0.007 & 0.46 $\pm$ 0.006 \\ 
        \midrule
        OVP & FBSR &  0.02 $\pm$ 0.001 & 0.02 $\pm$ 0.001 & 0.03 $\pm$ 0.002 & 0.05 $\pm$ 0.002 & 0.07 $\pm$ 0.004 & 0.09 $\pm$ 0.004 & 0.13 $\pm$ 0.004 \\
        & RC  &  0.02 $\pm$ 0.001 & 0.06 $\pm$ 0.003 & 0.1 $\pm$ 0.008 & 0.15 $\pm$ 0.01 & 0.2 $\pm$ 0.006 & 0.24 $\pm$ 0.01 & 0.29 $\pm$ 0.015 \\
        \midrule
        RS & FBSR &  0.02 $\pm$ 0.001 & 0.02 $\pm$ 0.001 & 0.03 $\pm$ 0.002 & 0.04 $\pm$ 0.002 & 0.06 $\pm$ 0.004 & 0.09 $\pm$ 0.003 & 0.12 $\pm$ 0.004 \\
        & RC &  0.02 $\pm$ 0.001 & 0.05 $\pm$ 0.002 & 0.09 $\pm$ 0.004 & 0.13 $\pm$ 0.005 & 0.17 $\pm$ 0.002 & 0.21 $\pm$ 0.005 & 0.25 $\pm$ 0.006 \\
        \bottomrule
        \end{tabular}
        \end{adjustbox}
        \label{tab:1}
    \end{table}
    
    \begin{table}
        \captionsetup{justification=centering}
        \caption{Kendall's Tau Correlation for the FBSR and RC algorithms on Synthetic Data when BF is varied. FOV = Fixed Order Vote, OV = Opposite Order Vote, OVP = Opposite Vote Probabilistic, RS = Random Subset}
        \centering
        \begin{adjustbox}{max width=1.15\textwidth,center}
        \begin{tabular}{c|c|ccccccc}
        \toprule
        Strategy & Algo & 0 & 0.05 & 0.1 & 0.15 & 0.2 & 0.25 & 0.3 \\
        \midrule
        \midrule
        FOV &  FBSR & 0.98 $\pm$ 0.003 & 0.97 $\pm$ 0.003 & 0.96 $\pm$ 0.003 & 0.92 $\pm$ 0.007 & 0.89 $\pm$ 0.012 & 0.85 $\pm$ 0.015 & 0.8 $\pm$ 0.018 \\
        & RC &  0.97 $\pm$ 0.003 & 0.92 $\pm$ 0.006 & 0.85 $\pm$ 0.01 & 0.77 $\pm$ 0.016 & 0.7 $\pm$ 0.014 & 0.62 $\pm$ 0.027 & 0.55 $\pm$ 0.026 \\
        \midrule
        OV &  FBSR &  0.98 $\pm$ 0.002 & 0.98 $\pm$ 0.002 & 0.98 $\pm$ 0.003 & 0.97 $\pm$ 0.002 & 0.97 $\pm$ 0.004 & 0.97 $\pm$ 0.004 & 0.96 $\pm$ 0.005 \\ 
        & RC &  0.97 $\pm$ 0.002 & 0.97 $\pm$ 0.002 & 0.94 $\pm$ 0.021 & 0.89 $\pm$ 0.018 & 0.72 $\pm$ 0.052 & 0.49 $\pm$ 0.079 & 0.15 $\pm$ 0.107 \\
        \midrule
        OVP &  FBSR &  0.98 $\pm$ 0.003 & 0.98 $\pm$ 0.002 & 0.98 $\pm$ 0.003 & 0.97 $\pm$ 0.003 & 0.97 $\pm$ 0.002 & 0.97 $\pm$ 0.004 & 0.96 $\pm$ 0.004 \\
        & RC &  0.98 $\pm$ 0.002 & 0.97 $\pm$ 0.002 & 0.97 $\pm$ 0.004 & 0.96 $\pm$ 0.004 & 0.96 $\pm$ 0.003 & 0.95 $\pm$ 0.005 & 0.93 $\pm$ 0.005 \\
        \midrule
        RS &   FBSR &  0.98 $\pm$ 0.002 & 0.98 $\pm$ 0.002 & 0.97 $\pm$ 0.003 & 0.97 $\pm$ 0.004 & 0.97 $\pm$ 0.004 & 0.96 $\pm$ 0.003 & 0.95 $\pm$ 0.009 \\
         & RC &  0.97 $\pm$ 0.002 & 0.97 $\pm$ 0.003 & 0.96 $\pm$ 0.006 & 0.94 $\pm$ 0.004 & 0.92 $\pm$ 0.008 & 0.88 $\pm$ 0.016 & 0.83 $\pm$ 0.027\\
        \bottomrule
        \end{tabular}
        \end{adjustbox}
        \label{tab:2}
    \end{table}

     \begin{table}
        \captionsetup{justification=centering}
        \caption{$L_2$-Errors for the FBSR  and RC algorithms on Synthetic Data when $n(=k)$ is varied.}
        \centering
        \begin{adjustbox}{max width=1.15\textwidth,center}
        \begin{tabular}{c|c|cccccc}
        \toprule
        Algo & BF & 50 & 90 & 130 & 170 & 210 & 250 \\
        \midrule
        \midrule
        FBSR & 0.1 &  0.063 $\pm$ 0.009 & 0.046 $\pm$ 0.005 & 0.042 $\pm$ 0.003 & 0.039 $\pm$ 0.005 & 0.037 $\pm$ 0.004 & 0.034 $\pm$ 0.003 \\
        FBSR & 0.2 &  0.133 $\pm$ 0.019 & 0.123 $\pm$ 0.011 & 0.106 $\pm$ 0.011 & 0.099 $\pm$ 0.008 & 0.1 $\pm$ 0.009 & 0.094 $\pm$ 0.006  \\
        \midrule
        RC & 0.1 &  0.14 $\pm$ 0.011 & 0.14 $\pm$ 0.005 & 0.14 $\pm$ 0.005 & 0.14 $\pm$ 0.005 & 0.14 $\pm$ 0.007 & 0.14 $\pm$ 0.002 \\
        RC & 0.2 &  0.27 $\pm$ 0.013 & 0.27 $\pm$ 0.005 & 0.26 $\pm$ 0.01 & 0.27 $\pm$ 0.006 & 0.27 $\pm$ 0.007 & 0.27 $\pm$ 0.005\\
        \bottomrule
        \end{tabular}
        \end{adjustbox}
        \label{tab:3}
    \end{table}
    
    \begin{table}
        \captionsetup{justification=centering}
        \caption{Kendall's Tau Correlation for the FBSR and RC algorithms on Synthetic Data when $n(=k)$ is varied.}
        \centering
        \begin{adjustbox}{max width=1.15\textwidth,center}
        \begin{tabular}{c|c|cccccc}
        \toprule
        Algo & BF & 50 & 90 & 130 & 170 & 210 & 250 \\
        \midrule
        \midrule
        FBSR & 0.1 &  0.936 $\pm$ 0.008 & 0.949 $\pm$ 0.008 & 0.952 $\pm$ 0.008 & 0.958 $\pm$ 0.007 & 0.958 $\pm$ 0.004 & 0.963 $\pm$ 0.004 \\
        FBSR & 0.2 &  0.867 $\pm$ 0.021 & 0.892 $\pm$ 0.017 & 0.899 $\pm$ 0.007 & 0.898 $\pm$ 0.014 & 0.902 $\pm$ 0.01 & 0.908 $\pm$ 0.008 \\
        \midrule
        RC & 0.1 &  0.83 $\pm$ 0.017 & 0.85 $\pm$ 0.02 & 0.85 $\pm$ 0.012 & 0.85 $\pm$ 0.014 & 0.84 $\pm$ 0.015 & 0.84 $\pm$ 0.005 \\
        RC & 0.2 &  0.67 $\pm$ 0.047 & 0.67 $\pm$ 0.027 & 0.69 $\pm$ 0.019 & 0.69 $\pm$ 0.021 & 0.69 $\pm$ 0.023 & 0.69 $\pm$ 0.01 \\
        \bottomrule
        \end{tabular}
        \end{adjustbox}
        \label{tab:4}
    \end{table}
    
    \begin{table}
        \captionsetup{justification=centering}
        \caption{$L_2$-Error for the BSR and RC algorithms on Sushi Dataset when the Byzantine Fraction is varied. FOV = Fixed Order Vote, OV = Opposite Order Vote, ORF = Opposite Random Flips}
        \centering
        \begin{adjustbox}{max width=1.15\textwidth,center}
        \begin{tabular}{c|c|ccccc}
        \toprule
         Strategy & Algo & 0 & 0.05 & 0.1 & 0.15 & 0.2 \\
        \midrule
        \midrule
         FOV & BSR & 0.0 $\pm$ 0.0 & 0.063 $\pm$ 0.013 & 0.102 $\pm$ 0.041 & 0.207 $\pm$ 0.062 & 0.267 $\pm$ 0.049 \\
         &RC &  0.0 $\pm$ 0.0 & 0.07 $\pm$ 0.019 & 0.15 $\pm$ 0.043 & 0.2 $\pm$ 0.047 & 0.24 $\pm$ 0.042 \\
         \midrule
         OV & BSR &  0.0 $\pm$ 0.0 & 0.08 $\pm$ 0.0 & 0.076 $\pm$ 0.0 & 0.126 $\pm$ 0.0 & 0.169 $\pm$ 0.0 \\
         &RC &  0.0 $\pm$ 0.0 & 0.11 $\pm$ 0.0 & 0.21 $\pm$ 0.0 & 0.29 $\pm$ 0.0 & 0.37 $\pm$ 0.0 \\
         \midrule
        ORF &BSR &   0.0 $\pm$ 0.0 & 0.135 $\pm$ 0.002 & 0.067 $\pm$ 0.001 & 0.085 $\pm$ 0.002 & 0.117 $\pm$ 0.006 \\
        &RC &   0.0 $\pm$ 0.0 & 0.09 $\pm$ 0.003 & 0.16 $\pm$ 0.002 & 0.23 $\pm$ 0.003 & 0.29 $\pm$ 0.003 \\
        \bottomrule
        \end{tabular}
        \end{adjustbox}
        \label{tab:5}
    \end{table}
    
    \begin{table}
        \captionsetup{justification=centering}
        \caption{Kendall's Tau Correlation for the BSR and RC algorithms on Sushi Dataset when Byzantine Fraction is varied. FOV = Fixed Order Vote, OV = Opposite Order Vote, ORF = Opposite Random Flips}
        \centering
        \begin{adjustbox}{max width=1.15\textwidth,center}
        \begin{tabular}{c|c|ccccc}
        \toprule
        Strategy & Algo & 0 & 0.05 & 0.1 & 0.15 & 0.2 \\
        \midrule
        \midrule
        FOV & BSR &  1.0 $\pm$ 0.0 & 0.942 $\pm$ 0.049 & 0.844 $\pm$ 0.116 & 0.742 $\pm$ 0.067 & 0.653 $\pm$ 0.084 \\
        &RC  &  1.0 $\pm$ 0.0 & 0.92 $\pm$ 0.062 & 0.88 $\pm$ 0.056 & 0.8 $\pm$ 0.069 & 0.69 $\pm$ 0.153 \\
        \midrule
        OV & BSR &  1.0 $\pm$ 0.0 & 1.0 $\pm$ 0.0 & 0.956 $\pm$ 0.0 & 0.822 $\pm$ 0.0 & 0.511 $\pm$ 0.0 \\
        & RC  &  1.0 $\pm$ 0.0 & 1.0 $\pm$ 0.0 & 0.96 $\pm$ 0.0 & 0.82 $\pm$ 0.0 & 0.38 $\pm$ 0.0 \\
        \midrule
        ORF & BSR &  1.0 $\pm$ 0.0 & 1.0 $\pm$ 0.0 & 0.991 $\pm$ 0.013 & 0.969 $\pm$ 0.018  & 0.893 $\pm$ 0.04 \\
        & RC &  1.0 $\pm$ 0.0 & 1.0 $\pm$ 0.0 & 0.96 $\pm$ 0.0 & 0.96 $\pm$ 0.013 & 0.85 $\pm$ 0.022 \\
        \bottomrule
        \end{tabular}
        \end{adjustbox}
        \label{tab:6}
    \end{table}
	
\end{document}